\documentclass[conference]{IEEEtran}
\IEEEoverridecommandlockouts

\usepackage{cite}
\usepackage{amsmath,amssymb,amsfonts,amsthm,algpseudocode}
\usepackage{thmtools, thm-restate}
\usepackage{algorithm}
\usepackage{graphicx}
\usepackage{textcomp}
\usepackage{xcolor}
\usepackage{float}
\usepackage{proof}
\usepackage{url}
\usepackage{booktabs}
\usepackage{comment}
\def\BibTeX{{\rm B\kern-.05em{\sc i\kern-.025em b}\kern-.08em
    T\kern-.1667em\lower.7ex\hbox{E}\kern-.125emX}}

\newcommand{\Gr}{G}
\newcommand{\Vx}{V}
\newcommand{\Edge}{E}

\newcommand{\A}{\mathbf{A}}

\newcommand{\Lap}{\boldsymbol{L}}
\newcommand{\D}{\boldsymbol{D}}

\newcommand{\bpsi}{\boldsymbol{\psi}}

\newcommand{\Id}{\boldsymbol{I}}

\newcommand{\tr}{\text{tr}}

\newcommand{\fn}{\boldsymbol{f}}
\newcommand{\gn}{\boldsymbol{g}}

\newcommand{\M}{\boldsymbol{M}}
\newcommand{\B}{\boldsymbol{B}}
\newcommand{\Samp}{\Omega}

\newtheorem{Theorem}{Theorem}
\newtheorem{Proposition}{Proposition}
\newtheorem{Definition}{Definition}
\theoremstyle{empty}
\newtheorem{duplicate}{Theorem}

\begin{document}

\title{Bayesian Spectral Graph Denoising with Smoothness Prior
}
\makeatletter 
\newcommand{\linebreakand}{%
  \end{@IEEEauthorhalign}
  \hfill\mbox{}\par
  \mbox{}\hfill\begin{@IEEEauthorhalign}
}
\DeclareRobustCommand*{\IEEEauthorrefmark}[1]{%
  \raisebox{0pt}[0pt][0pt]{\textsuperscript{\footnotesize #1}}%
}
\author{\IEEEauthorblockN{Sam Leone\IEEEauthorrefmark{1$\dag$},
Xingzhi Sun\IEEEauthorrefmark{2$\dag$}, 
Michael Perlmutter\IEEEauthorrefmark{3} and
Smita Krishnaswamy\IEEEauthorrefmark{1,2,4,5,6,7,8}}
\IEEEauthorblockA{\small \IEEEauthorrefmark{1}Program for Applied Mathematics, Yale University \quad \IEEEauthorrefmark{2}Department of Computer Science, Yale University}
\IEEEauthorblockA{\small \IEEEauthorrefmark{3}Department of Mathematics, Boise State University \quad \IEEEauthorrefmark{4}Department of Genetics, Yale School of Medicine}

\IEEEauthorblockA{\small \IEEEauthorrefmark{5}Department of Genetics, Yale School of Medicine \quad \IEEEauthorrefmark{6}Wu Tsai Institute, Yale University \quad \IEEEauthorrefmark{7}FAIR, Meta AI}

\IEEEauthorblockA{\small \IEEEauthorrefmark{8}Computational Biology and Bioinformatics Program, Yale University \quad \IEEEauthorrefmark{$\dag$}Equal Contribution}}


\maketitle

\begin{abstract}
Here we consider the problem  of denoising features associated to complex data, modeled as signals on a graph, via a smoothness prior. This is motivated in part by settings such as single-cell RNA where the data is very high-dimensional, but its structure can be captured via an affinity graph.  This allows us to utilize ideas from graph signal processing. In particular, we present algorithms for the cases where the signal is perturbed by Gaussian noise, dropout, and uniformly distributed noise. The signals are assumed to follow a prior distribution defined in the frequency domain which favors signals which are smooth across the edges of the graph. By pairing this prior distribution with our three models of noise generation, we propose \textit{Maximum A Posteriori} (M.A.P.) estimates of the true signal in the presence of noisy data and provide algorithms for computing the M.A.P. Finally, we demonstrate the algorithms' ability to effectively restore signals from white noise on image data and from severe dropout in single-cell RNA sequence data.
\end{abstract}


\section{Introduction}
Signals defined on modern large-scale, irregularly structured  data sets are often corrupted by large amounts of noise such as measurement error or missing measurements. This motivates one to estimate the most likely \emph{true, uncorrupted} values of the signal based on both the noisy observations and their prior beliefs about the signal, which often takes the form of a smoothness assumption. We shall present an approach for producing such \emph{Maximum A Posteriori} (M.A.P.) estimates which utilizes tools from spectral graph theory. 

\par Our method is motivated by the explosion in recent decades of complex high-dimensional data, and associated signals, with very high noise levels. Such data may explicitly reside on a graph, e.g., social, energy, transportation, sensor, or neuronal networks\cite{shuman2013emerging}, or it may implicitly have relationships between entities from which a graph can be built, for example, physical and biological systems, text, image, time series \cite{zhou2020graph,moon2018manifold,van2018recovering,moon2019visualizing,coifman2006diffusion}
%
\par With the graph (either existing or built from data), we can treat features as signals (functions) on the graph, and apply methods in graph signal processing, especially spectral graph theory. 
Typically, a well-behaved signal defined on the vertices will take similar values at vertices that are more connected. This leads us to the prior that many functions of interest will be \emph{smooth} on the graph, where the concept of smoothness can be quantified using tools from spectral graph theory and the eigendecomposition of the graph Laplacian. 
This intuition motivates the following approach. First, we assume a priori that the signal of interest is likely ``fairly smooth" on the graph. Then, we model the noise of the observations. Finally, we produce an estimate of the true signal with the highest likelihood based on our prior beliefs and the observed measurements. Importantly, we note that the assumption that the signal is smooth is  very mild and we do not assume that the signal (or the data on which it is defined) has any specific form. We provide details on how to implement this approach under several noise models and then demonstrate the effectiveness of our method on real-world and synthetic data sets.
We also note that our method fills the gap of theoretical guarantees in the popular method MAGIC\cite{van2018recovering}, which it outperforms  due to the specific modeling of noise types.
\section{Background \& Related Work}\label{sec: notation}
\subsection{An example for high-dimensional data}
We first motivate our method via an example of denoising features associated to complex high-dimensional  data. Single-cell RNA sequence (scRNA-seq) provides high resolution information about gene expression and is of great interest in molecular biology, clinical studies, and biotechnology \cite{saliba2014single,hwang2018single}. scRNA-seq data is high-dimensional, as it measures the expression of tens of thousands of genes on up to millions of cells\cite{jovic2022single}, and  suffers from high noise levels due to multiple sources. Reducing this noise is a crucial step, which is needed prior to downstream analysis\cite{kolodziejczyk2015technology}.
\par In single-cell RNA sequence data, one obtains the gene-expression counts for a variety of genes in each cell. Each cell can then be viewed as a high-dimensional vector (whose $i$-th coordinate corresponds to the amount of gene $i$ expressed). It is a common practice to turn this data, consisting of  high-dimensional vectors (cells), into a graph by placing edges between cells which are close together in high-dimensional space, and viewing the expression of each gene, as a signal (function) defined on the cellular graph\cite{levine2015data,van2018recovering,butler2018integrating,wang2021scgnn}.
\subsection{Graph Signal Processing with Bayesian inference}
Spectral graph theory concerns itself with the distribution of eigenvalues and eigenvectors of 
matrices associated with graphs. The  set of eigenvalue-eigenvector pairs is known to uncover the geometric and algebraic properties of a graph. This is the observation that drives algorithms like spectral clustering and diffusion maps \cite{coifman2006diffusion}, the main intuition being that low frequency eigenpairs capture low-resolution, key information about a graph's structure.

Graph Signal Processing (GSP) utilizes tools from spectral graph theory to extend the Fourier transform from classical signal processing and time series analysis to the graph setting \cite{shuman2013emerging}. In the classical methods, 
signals can be denoised by mapping the signal to the Fourier domain, reducing the high-frequency component of the function, and inverting the Fourier 
transform to achieve a ``smoother" version of the signal. In much the same way, GSP operates by 
representing graph signals in a basis eigenvectors for the graph Laplacian (defined below), whose corresponding eigenvalues may be interpreted as (squared) frequencies, 
and then reducing the high-frequency components. Filtering in this manner has been applied to use cases such as 
point cloud data, biological imaging, sensor networks, and more \cite{ortega2018graph}. 

Bayesian inference is a fundamental method of parameter estimation. The typical form of this problem is that there is a random variable $\boldsymbol{x}$ drawn from some prior distribution and another random variable $\boldsymbol{y}$  whose distribution depends on $\boldsymbol{x}$. The ambition of Bayesian estimation is, given only $\boldsymbol{y}$, to estimate the underlying value of $\boldsymbol{x}$ using both prior information on $\boldsymbol{x}$ and the interaction between $\boldsymbol{x}$ and $\boldsymbol{y}$.
\par Notably, two important, nonstandard aspects of our method are: (1) we do not have any explicit prior on a data on which the signals is defined, but rather directly build the graph from the data (if it does not already exist), and treat it as a deterministic structure; (2) we do not assume the signal has any specific form, but rather use a mild prior of its smoothness on the graph. These distinctions free us from the limitations of Bayesian models caused by model misspecification\cite{liang2020model}, and make our method generally applicable to the vast range of data sets regardless of the data distributions. 

\subsection{MAGIC: Markov Affinity-based Graph Imputation of Cells}
\par MAGIC \cite{van2018recovering} is a commonly used method for denoising single-cell data. It is based on the idea that the high-dimensional data lies on a low-dimensional manifold, represented by a nearest neighbor graph. After building the graph, it uses data diffusion, which is random-walk on the graph to denoise the model. Its has been tremendously successful; however, it lacks a solid theoretical model. Our method fills this gap with GSP and Bayesian inference. Furthermore, by specifying cases of common noise models, we are able to adjust our model accordingly, allowing us to outperform MAGIC in these cases.
\subsection{Notation and Defininitions}
Throughout, we shall let $\Gr = (\Vx,\Edge,w)$ denote a weighted, connected, and undirected graph  with $|\Vx|=n$  and $|\Edge|=m$. Without loss of generality, we  will assume 
$\Vx = \{1,\ldots,n\}$. We shall refer to functions $\fn : \Vx \to \mathbb R$ as \emph{graph signals}. In a slight abuse of notation, we will not distinguish between $\fn$ and the vector in $\mathbb R^n$ whose $a$-th entry is $\fn(a)$, $1\leq a\leq n$. We shall let $\A$ denote the \emph{weighted} adjacency matrix and let $\D=\text{diag}(\A\boldsymbol{1})$ denote the corresponding diagonal degree matrix.

Given $\A$ and $\D$, the \emph{combinatorial Laplacian} is defined as $\Lap = \D - \A$.  Is is well-known that $\Lap$ admits an orthonormal basis of eigenvectors, $\Lap\bpsi_i=\lambda_i\bpsi_i,$ $1\leq i \leq n,$ where $\bpsi_1=\boldsymbol{1}$ and $0=\lambda_1<\lambda_2\leq\ldots\leq \lambda_n$. It follows that $\Lap$ is a positive semi-definite matrix whose null space is equal to $\textbf{span}\{\boldsymbol{1}\}$. One may compute that the quadratic form corresponding to $\Lap $ is given by $\fn^\top \Lap \fn = \sum_{\{a,b\} \in E} w(a,b) (\fn(a)-\fn(b))^2$. Thus, setting $\fn=\bpsi_i$, the $\lambda_i$ are interpreted as (squared) frequencies, representing the rate at which $\bpsi_i$ oscillates over the graph, and the $\bpsi_i$ are interpreted as \emph{generalized Fourier modes}, where $\widehat{\fn}(\lambda_i)=\langle \fn,\bpsi_i\rangle$ represents the portion of $\fn$ at frequency $\lambda_i$. Since the $\bpsi_i$ are an orthonormal basis, we have $\fn = \sum_{i=1}^n \widehat{\fn}(\lambda_i)\bpsi_i$. Therefore, for a real-valued function $h$, we can define a corresponding filter by $h(\fn) = \sum_{i=1}^n h(\lambda_i)\widehat{\fn}(\lambda_i)\bpsi_i$. 
 We shall let $\B$ denote the weighted $m\times n$ incidence matrix, where rows correspond to edges and columns to vertices, whose entries are given by  $\B(e,a) = - \B(e,b) = \sqrt{w(a,b)}$, if $e=(a,b)$ and $\B(e,v) = 0$ for all $v \not \in \{a,b\}$ \cite{spielman2012spectral}. One may verify that the Laplacian can be factored as $\Lap = \B^\top\B$. (Here, we implicitly assume an arbitrary, but fixed, ordering of each edge $(a,b)$. This arbitrary choice does not affect the identity $\Lap = \B^\top\B$ nor any of our analysis.)

We shall let $p(\fn)$ denote the probability distribution of a random variable $\fn$ and shall let $p(\fn|\gn)$ denote the conditional distribution of $\fn$ given another random variable $\gn$. We shall make use of the fact that by Bayes' theorem, $p(\fn|\gn)\propto p(\fn)p(\gn|\fn)$, where $\propto$ denotes proportionality and the implied constant depends on $\gn$.

\section{Methods}
Our goal is to estimate an unknown signal $\fn\in\mathbb{R}^n$ based on an  observation $\gn\in\mathbb{R}^n$, which we interpret as a noisy version of $\fn$ under various settings. In each case, we will assume that an observed $\gn$ is obtained from a corruption of a true signal $\fn$ which lies within a corresponding admissibility class $\Samp_{\gn}$. We shall then define the \emph{maximum a posteriori} estimate of $\fn$ to be the most likely value of $\fn$ based on (i) the fact that $\gn$ was observed and (ii) our \emph{a priori} beliefs on $\fn$ discussed in the following subsection. 
\subsection{A Prior Distribution Based on Smoothness}\label{sec: prior}

We define prior distributions on $\widehat{\fn}(\lambda_i)$ for  $i = 2,\ldots, n$, assuming that each $\widehat{\fn}(\lambda_i)$ follows the probability distribution:
\[ p_\kappa(\boldsymbol{\widehat{\fn}}(\lambda_i))  \propto \exp\big(-\kappa \lambda_i \boldsymbol{\widehat{f}}(\lambda_i)^2\big) \]
where $\kappa$ is a fixed smoothing parameter. We further assume that the  $\boldsymbol{\widehat{f}}(\lambda_i)$ are independent which implies that, for any fixed value of $\widehat{\fn}(\lambda_1)$,  the probability distribution of $\widehat{\fn}$ satisfies
\begin{align*} p_\kappa(\boldsymbol{\widehat{f}})  \propto \prod_{i=2}^n \exp\big(-\kappa \lambda_i \boldsymbol{\widehat{f}}(\lambda_i)^2\big)   = \exp\bigg( - \kappa \sum_{i=2}^n \lambda_i \boldsymbol{\widehat{f}}(\lambda_i)^2 \bigg).  \end{align*}
We then give $\fn$ the probability distribution defined by taking the inverse GFT of $\widehat{\fn}$.
We note that since the $\bpsi_i$ are an orthonormal eigenbasis and $\fn = \sum_{i=1}^n \widehat{\fn}(\lambda_i)\bpsi_i$,  we have
\[ p_\kappa(\fn) \propto \exp\bigg( - \kappa \sum_{i=2}^n \lambda_i \boldsymbol{\widehat{f}}(\lambda_i)^2 \bigg) = \exp \big(-\kappa \fn^\top \Lap \fn \big). \]
Therefore, we see that this probability distribution is defined so that the likelihood of $\fn$ decreases with its variation across the graph and $\kappa$ acts as a parameter controlling the tolerance towards fluctuation. 
Notably, we do not assume any prior distribution on $\widehat{\fn}(\lambda_1)$ (although is some cases $\widehat{\fn}(\lambda_1)$ will be implicitly constrained by the admissibility class $\Samp_{\gn}$). Therefore, our maximum a posteriori estimate is simply the most likely value of $\fn$ based on the $\gn$ and our prior beliefs about $\widehat{\fn}(\lambda_2),\ldots,\widehat{\fn}(\lambda_n)$. 

\subsection{Gaussian Noise on the Graph}
We first consider the setting where each of the Fourier coefficients is corrupted by Gaussian noise, i.e., $\widehat{\gn}(\lambda_i) = \boldsymbol{\widehat{f}}(\lambda_i) + z_i$, $2 \leq i \leq n$, where each $z_i \sim \mathcal N(0,\sigma^2)$ is an independent normal random variable. We will further assume that the total noise $\gn-\fn$ has zero mean, which motivates us to define the admissibility class $\Omega_{\gn} = \{ \fn : \boldsymbol{\widehat{f}}(\lambda_1) =  \widehat{\gn}(\lambda_1)\}$. By expanding the conditional and a priori densities and utilizing the fact that for a given $\gn$, we have $p_\kappa(\fn | \gn)\propto p_\kappa(\fn)p_\kappa(\gn|\fn)$, one may derive a maximum a posteriori estimate of $\fn$ given $\gn$\footnote{Further details on the derivation of Theorem \ref{thm: Gaussian}, and all of our other theoretical results, are available at \url{https://arxiv.org/abs/2311.16378}}. 

\begin{Theorem}[Gaussian Denoising]\label{thm: Gaussian}
    Let $\gn$ be given, and let $\Omega_{\gn} = \{ \fn : \boldsymbol{\widehat{f}}(\lambda_1) =  \widehat{\gn}(\lambda_1)\}$. As above, assume that $\widehat{\gn}(\lambda_i) = \boldsymbol{\widehat{f}}(\lambda_i) + z_i, 2\leq i\leq n$, $z_i\sim\mathcal{N}(0,\sigma^2)$ and that our prior beliefs on $\fn$ are as described in Section \ref{sec: prior}. Then, the maximum a posteriori likelihood estimate of $\fn$ given $\gn$ is, 
    \[ \fn_{\text{map}} = h(\gn), \]
    where $h(\gn)$ is a filter as described in Section \ref{sec: notation} with $h(\lambda_i) = \frac{1}{1 + 2\kappa \sigma^2 \lambda_i}$.
    Moreover, $\fn_{\text{map}}$ can be computed, to within  $\epsilon$ accuracy in the $\Lap$-norm ($\|\fn\|_{\Lap}^2=\fn^\top\Lap\fn$), in time $\tilde{\mathcal O}\big(m \log(\epsilon^{-1}) \min\big\{\sqrt{\log(n)}, \sqrt{\frac{2\kappa\sigma^2\lambda_{\max} + 1}{2\kappa\sigma^2\lambda_{\min} + 1}} \big\}\big)$.
\end{Theorem}
We note that the minimum in the term describing the time complexity arises from the existence of two possible methods of computation, both of which are algorithms for solving linear systems in an implicit matrix $\M$ with a condition number of $\beta = \frac{2\kappa\sigma^2\lambda_{\max} + 1}{2\kappa\sigma^2\lambda_{\min} + 1}$. When $2\kappa\sigma^2$ is small, $\beta$ is small and the conjugate gradient algorithm will terminate rapidly. Alternatively, when $\beta$ is large, one may use the solver from  \cite{cohen2014solving} which requires 
$\tilde{\mathcal O}\big(m \log(\epsilon^{-1}) \sqrt{\log(n)}\big)$ time.

 In practice, $\sigma^2$ and $\kappa$ are generally unknown. However, as the filter depends only on the product $2\kappa\sigma^2$, it suffices to estimate this quantity, which we denote by $\tau$. We propose a method of moments estimator which calculates the expectation of $\gn^\top \Lap \gn, \gn^\top \Lap^2 \gn$ in terms of $\sigma, \kappa$ and backsolves using the empirical values.  Alternatively, we may regard $\tau$ as a smoothing parameter to be tuned, rather than  a quantity needing estimation. 
\begin{equation}\label{gauss_strength} \tau \approx \frac{tr(\Lap) \gn^\top \Lap \gn - (n-1) (\Lap \gn)^\top  (\Lap \gn)}{tr(\Lap)  (\Lap \gn)^\top  (\Lap \gn)  - tr(\Lap^2) \gn^\top \Lap \gn} \end{equation}
Note that, by the handshake lemma, $\tr(\Lap) = \sum_{a \in \Vx}\deg(a)= 2\sum_{(a,b) \in \Edge}w(a,b)$. Furthermore, $\tr(\Lap^2) = \sum_{a}\big(\deg(a)^2 + \sum_{(a,b) \in \Edge}w(a,b)^2 \big) $, and so both of these quantities can be calculated in $\mathcal O(m)$ time. 
Alternatively, we may regard $\tau$ as a smoothing parameter to be tuned, rather than  a quantity needing estimation. We note that this filter may be viewed as a form of Tikhonov regularization \cite{shuman2013emerging}.

\subsection{Uniformly Distributed Noise} 

Next, we consider the case when the noise is a random uniform scaling in the vertex domain: $\gn(a) = u(a) \fn(a)$, where each $u(a) \sim \text{Unif}[0,1]$ is an independent uniform random variable. In this case, since $0\leq u(a)\leq 1$, we set the admissibility class $\Samp_{\gn}=\{\fn: |\fn(a)|\geq |\gn(a)|,  \text{sign}(\fn) = \text{sign}(\gn),  \forall a\in V\}$. For such an $\fn \in \Omega_{\gn}$, one may compute that the a posteriori likelihood of $\fn\in\Samp_{\gn}$ given $\gn$ is 
\[ p_\kappa(\fn|\gn) = p_\kappa(\fn) \prod_{a \in V} \frac{1}{|\fn(a)|}. \]
 We will maximize the a posteriori likelihood by minimizing the negative log likelihood, which using basic properties of the logarithm leads us to the optimization problem
\[  \min_{\fn \in \Omega_{\gn}}\mathcal{L}(\fn),\quad \mathcal{L}(\fn)= \kappa \fn^\top \Lap \fn + \sum_{a \in V} \log |\fn(a)|. \]
In order to (approximately) solve this problem, we adopt a constrained Convex-Concave Procedure (CCP) \cite{lipp2016variations} for the above. The CCP operates by splitting a function of the form $f(x) = f_{\text{concave}}(x) + f_{\text{convex}}(x)$ and approximating the concave portion linearly about the current solution; the relaxed problem is convex and can be solved more efficiently. The procedure is repeated until convergence, and it is known to be a descent algorithm. Applied this particular optimization, the CCP update of $\fn^{t+1}$ from $\fn^t$ is as follows:
    \[ \fn^{t+1} = \arg\min_{\fn \in \Omega_{\gn}} \kappa \fn^\top \Lap \fn + \sum_{a \in V} \frac{\fn(a)}{|\fn^t(a)|} \]
We remark that $\fn^{t+1}$ can be computed as a quadratic program and that the update provides a descent algorithm - $\mathcal L(\fn^{t+1}) \leq \mathcal L(\fn^t)$. This is because the loss function is a quadratic function of $\fn$ and the feasible region $\Omega_{\gn}$ is a convex polyhedron.
\subsection{Partial Observations \& Bernoulli Dropout}
In our final two models, we consider two settings where the noise behaves differently at different vertices. We assume that there is some (possibly unknown) set $S\subseteq V$ where $f(a)$ is exactly equal to $g(a)$ for all $a\in S$. We make no assumption regarding the relationship between $f(a)$ and $g(a)$ for $a \notin S$. This leads us to define the admissibility class $\Omega_{\gn} = \{\fn : \fn(a) = \gn(a) \text{ for all } a \in S \}$. We consider two practically useful variations of this problem: 
\begin{enumerate}
    \itemsep 0em 
    \item Basic Interpolation: The set $S$ is known. 
    \item Bernoulli Dropout: There is a ``set of suspicion'' $\zeta$ where we are unsure whether $a \in S$  or $a \in S^c$. There is also a (possibly empty) set $\zeta^c$ for which the observer is certain of their observations (i.e., we know $\zeta^c\subseteq S$). For each $a\in \zeta$, we assume that $a$ is corrupted (i.e., $a\notin S$) with probability $p$ and that $a\in S$ with probability $1-p$.
\end{enumerate}
In this first scenario, the maximum a posteriori estimate of $\fn$ is the most likely $\fn$ that is equal to $\gn$ over the observation set $S$: $\fn_{\text{map}} = \arg\max_{f \in \Omega_{\gn}} p_\kappa(\fn)$. Because of the monotonicity of the exponential function, this is equivalent to computing $\min_{f \in \Omega_{\gn}}\fn^\top \Lap \fn$. This problem was studied in \cite{zhu2003semi}, which proved the following result. Notably, \cite{zhu2003semi} predated the development of  efficient solvers which could be used to compute $\fn_{map}$ as in \eqref{eqn: fmap formula}. However, now that such solvers exist\cite{spielman2004nearly}, one may use them to compute the proposed estimate to accuracy $\epsilon$ in $\tilde{\mathcal O}(\widehat{m}\sqrt{\log\widehat{n}}\log(\epsilon^{-1}))$ time, where $\partial S$ is the boundary of $S$, $\widehat{n} = |S^c \cup \partial S|$ and $\widehat{m} = |E(S^c,S^c) \cup E(S^c, S)|$, where $E(S_1,S_2)$ denotes the set of edges going from $S_1 \subseteq \Vx$ to $S_2 \subseteq \Vx$. We also denote $\fn(A):=(\fn(a_1),\fn(a_2),\dots,\fn(a_k))$, where $\{a_1,a_2,\dots,a_k\}= A \subseteq V$; $\Lap(S_1,S_2)$ and $\A(S_1,S_2)$ are the restrictions of $\Lap$ and $\A$ to rows in $S_1$ and columns in $S_2$, respectively; $\B(:,S_1)$ is the restriction of $\B$ to columns in $S_1$; $\forall S_1, S_2\subseteq V$.

\begin{Theorem}[Restated from \cite{zhu2003semi}]\label{thm: S known}
    Suppose $S$ has at least one edge going to $S^c$. Then there exists a unique solution to $\min_{\fn \in \Omega}\fn^\top\Lap \fn$. The interpolation of $\fn$ to $S^c$ is given by
    \begin{equation}\label{eqn: fmap formula} \fn_{\text{map}}(S^c) = \Lap(S^c,S^c)^{-1}\A(S^c,S)\gn(S). 
    \end{equation}
\end{Theorem}

Now, we consider the second, more challenging scenario where we observe a signal $\gn$ which is equal to $\fn$, except in a set of suspicion $\zeta$ where, with probability $p$, $g(a)$ is corrupted (i.e., not equal to $f(a)$). Based on the observation $\gn$ alone, there is no obvious way to identify the set $S=\{a\in V: \fn(a)=\gn(a)\}$ (although we do know $\zeta^c\subseteq S)$. However, we  note that for $\gn$ to take a given value, there must be $\|\fn(\zeta)-\gn(\zeta)\|_0$ corrupted entries and $|\zeta|-\|\fn(\zeta)-\gn(\zeta)\|_0$ uncorrupted entries. Since each entry is corrupted with probability $p$, we model $p_\kappa(\gn|\fn)\propto p^{\|\fn(\zeta)-\gn(\zeta)\|_0}(1-p)^{|\zeta| - \|\fn(\zeta)-\gn(\zeta)\|_0}$. Thus, for $\fn\in\Samp_{\gn}$, the negative log of the posterior likelihood  of $\fn$ can be estimated as:
\begin{align*} -\log p_\kappa(\fn | \gn) &= \kappa \fn^\top \Lap \fn  \\ & +  \|\fn(\zeta) - \gn(\zeta)\|_0 \big(\log(1-p)-\log(p) \big)  \\ & + \text{constant}. \end{align*}
Therefore, if we define $\tau = \kappa^{-1}(\log(1-p)-\log(p))$, we observe the MAP is produced by the following minimization problem: 
\[ \fn_{\text{map}} \in \min_{\fn \in \Omega_{\gn}}\fn^\top \Lap \fn + \tau \|\fn(\zeta)-\gn(\zeta)\|_0. \]
Note that the sign of $\tau$ is going to depend on $\log(1-p)-\log(p) = \log(\frac{1}{p}-1)$. When $p < 1/2$, then the penalty term $\tau$ is positive; otherwise, $\tau < 0$ so we may assume all values have changed and estimate $\fn$ using Theorem \ref{thm: S known}. When $p < 1/2$, the penalty term is positive. By breaking up $\fn$ into $\fn(\zeta)$ and $\fn(\zeta^c)$, we may write the optimization as a regression problem:
\begin{Theorem}
When $p < 1/2$, the $\fn_{\text{map}}$ is the $\arg\min$ of the following sparse regression problem: 
        \begin{align*} \fn_{\text{map}}(\zeta) & \in   \gn(\zeta) + \arg\min_{\boldsymbol{x}} \big\{ \tau\|\boldsymbol{x}\|_0  \\
    & + \|\B(:,\zeta)\boldsymbol{x} - \B(:,\zeta^c)\gn(\zeta^c) + \B(:,\zeta)\gn(\zeta)\|_2^2 \big\}.  \end{align*}
And when $p \geq 1/2$, the solution is given by, 
\[ \fn_{\text{map}}(\zeta) = \Lap(\zeta^c,\zeta^c)^{-1}\A(\zeta^c, \zeta)\gn(\zeta^c). \]
\end{Theorem}
\begin{figure}
\centering
    \includegraphics[width=0.15\linewidth]{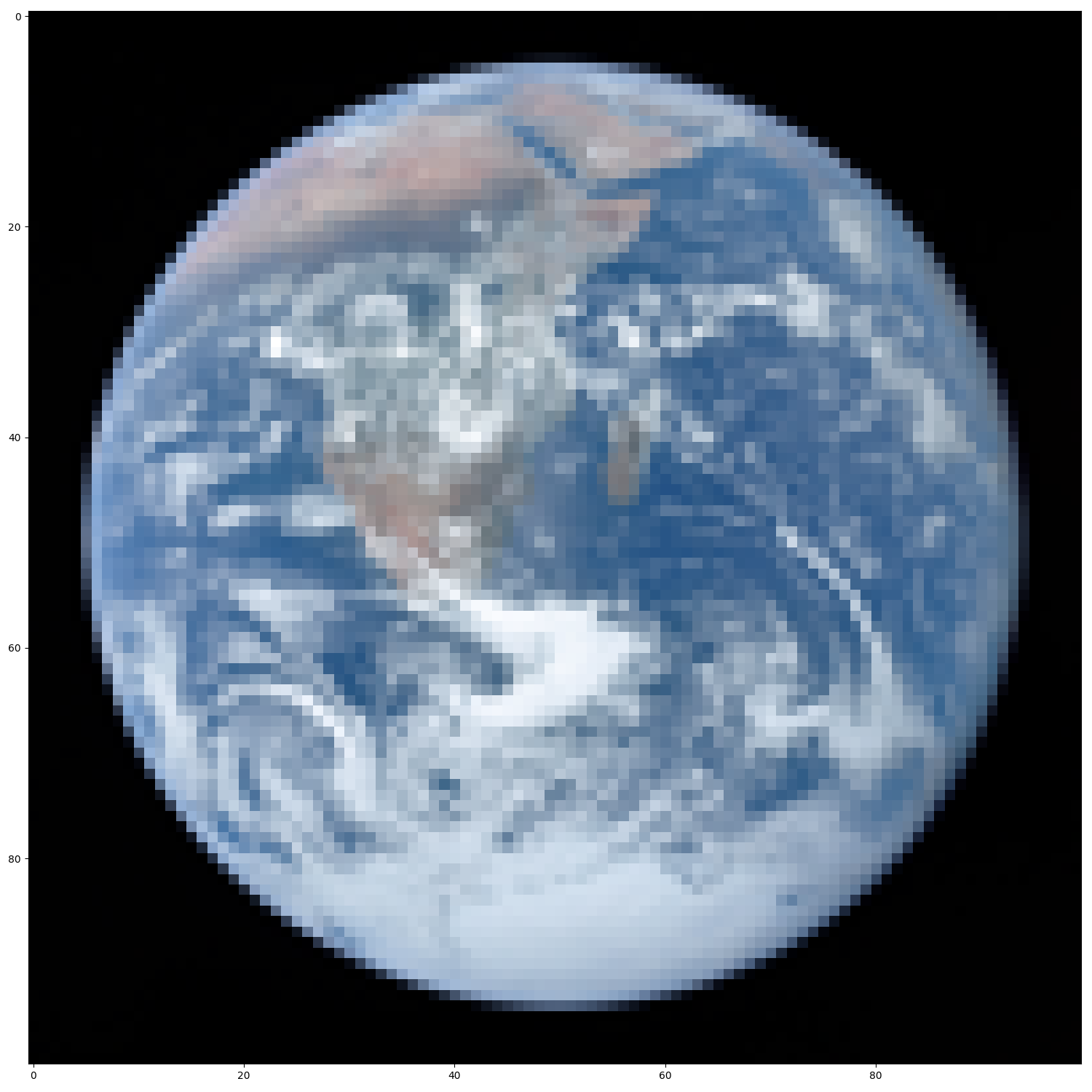}
    \includegraphics[width=0.15\linewidth]{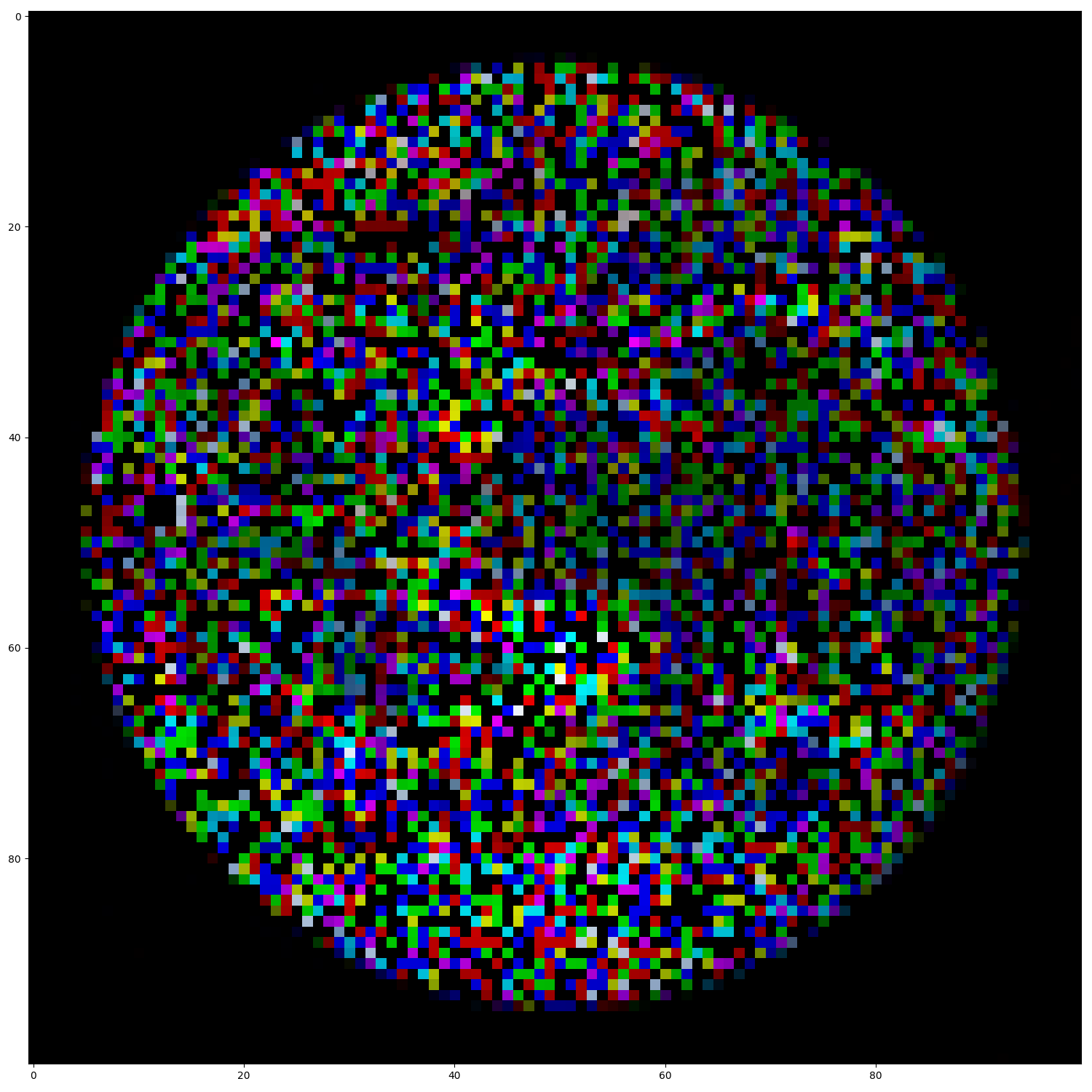}
    \includegraphics[width=0.15\linewidth]{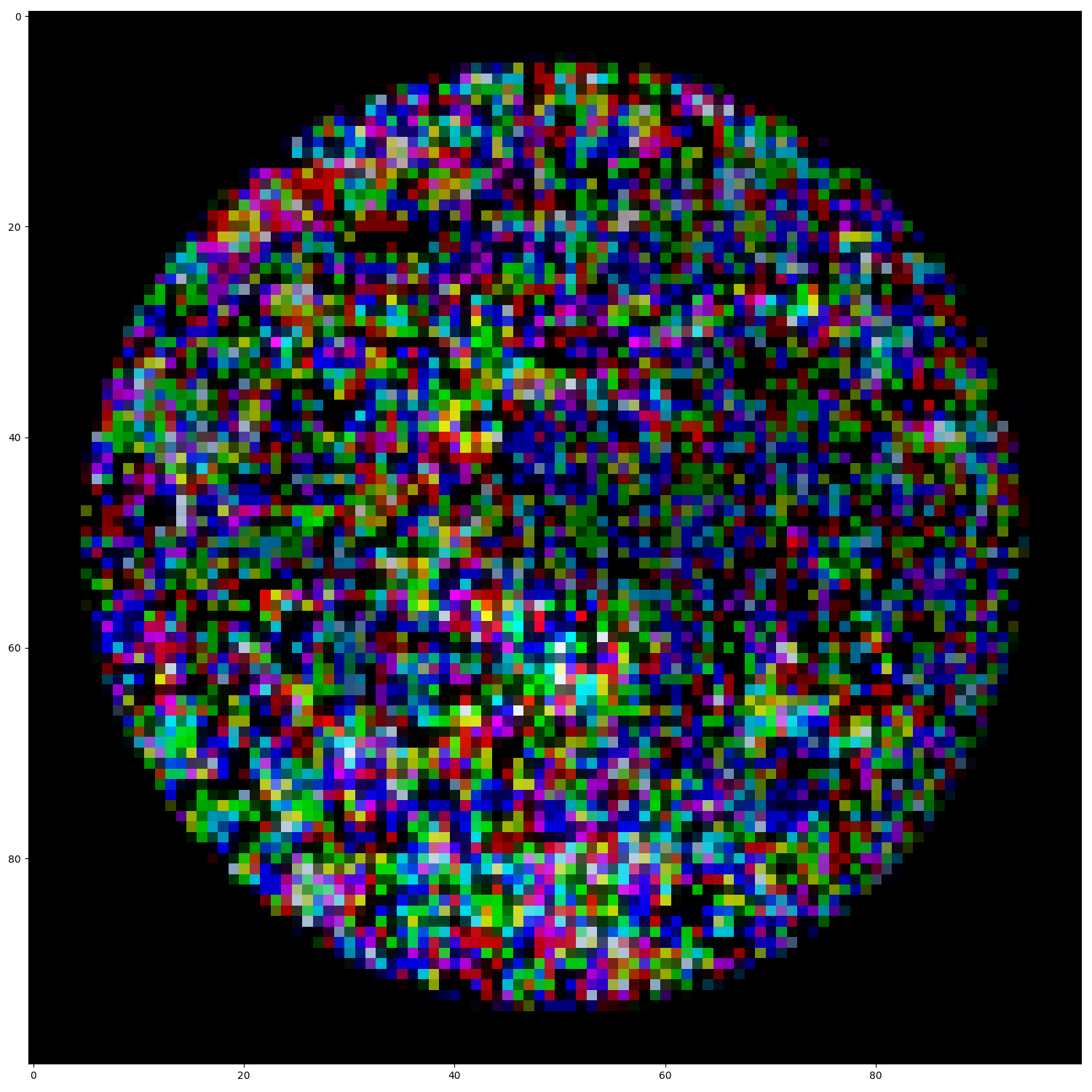}
    \includegraphics[width=0.15\linewidth]{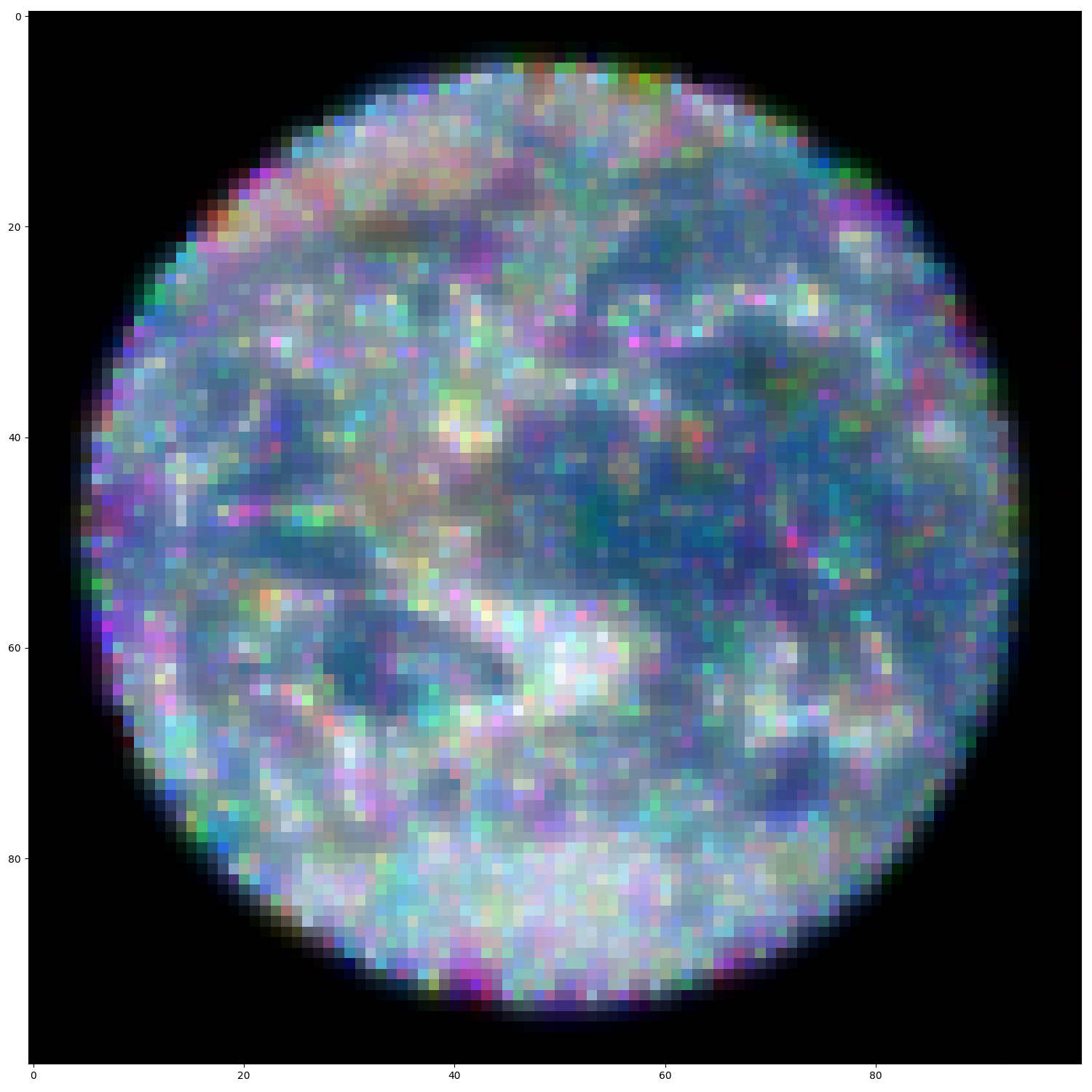}
    \includegraphics[width=0.15\linewidth]{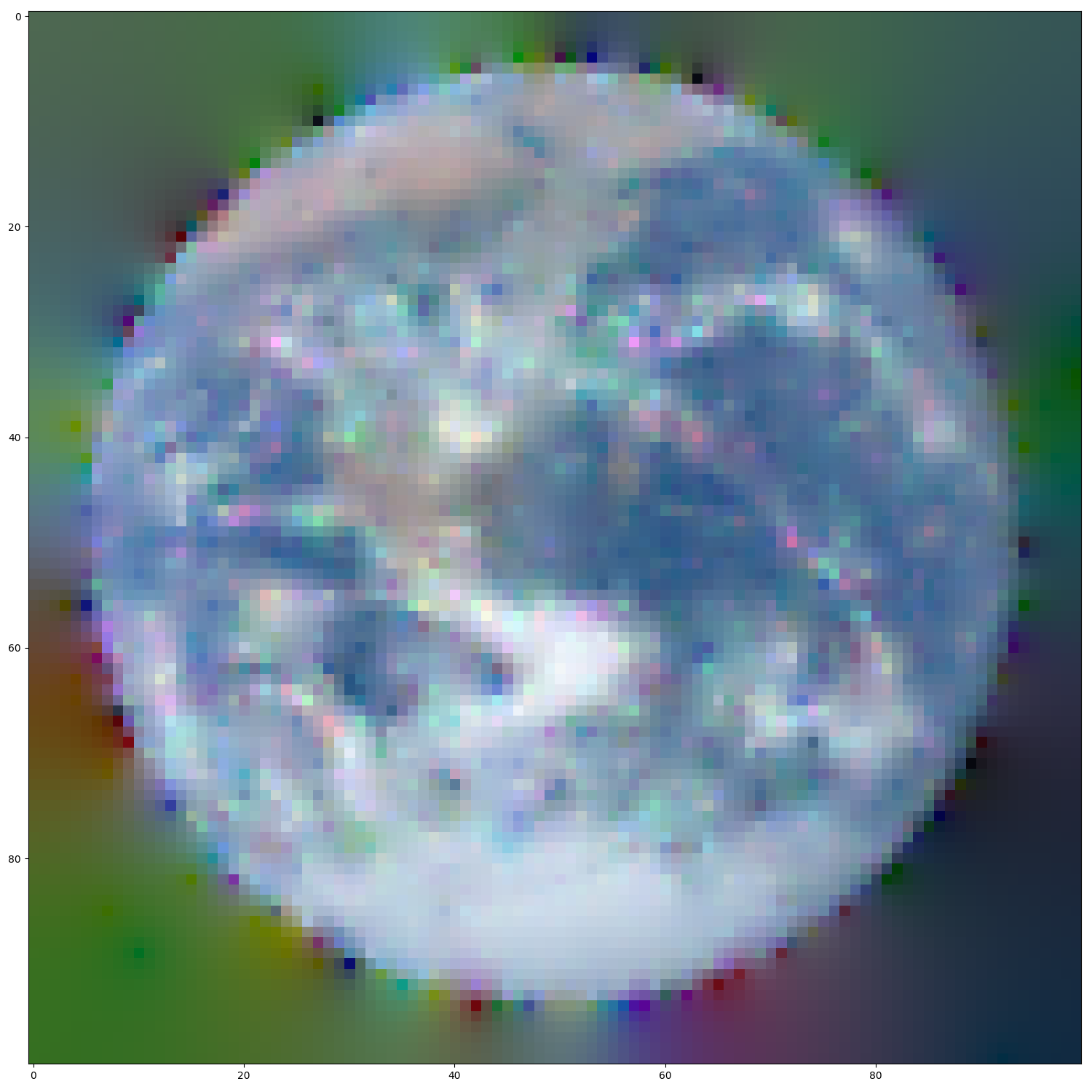}
    \caption{The results of LASSO regularization along with different parameters of $\tau$. In this case, the region of skepticism is the set of zeroes $\zeta = \{a \in \Vx : \gn(a) = 0\}$. The leftmost image is ground truth, the second image is the corrupted signal (i.i.d. across each pixel and channel with dropout probability $p = 0.7$). The last three images are Bernoulli-LASSO restorations with $\tau = 10^{-2}, 10^{-3}$, and $\tau = 0$.}
\end{figure}

In general, the problem of $\ell_0$-regularized regression is NP-Hard \cite{natarajan1995sparse}
. However, numerous approximation methods exist including branch and bound \cite{hazimeh2020sparse} and an $\ell_2$-based greedy algorithm \cite{natarajan1995sparse}. Alternatively, we may consider a relaxed version of the minimization problem in which the $\ell_0$ penalty term is replaced with an $\ell_1$ penalty term. In this case, the relaxed $\fn_{\text{map}}$ can be found via LASSO regression \cite{tibshirani1996regression}, for which many efficient algorithms exist. 

Finally, we draw special attention to the ``no-trust" case where $\zeta = \Vx$, i.e. we are skeptical of all observations. Then the optimization can be written more simply: 
 \begin{align*} \fn_{\text{map}}(\zeta) & \in \gn(\zeta) + \arg\min_{\boldsymbol{x}}\|\B\boldsymbol{x}-\B\boldsymbol{g}\|_2^2 + \tau \|\boldsymbol{x}\|_0
\end{align*}
The benefit of the no-trust estimate is that it makes few assumptions about the nature of the noise and does not require the user to come up with $\zeta$. 

\begin{figure}
    \centering
    \includegraphics[width=0.2\linewidth]{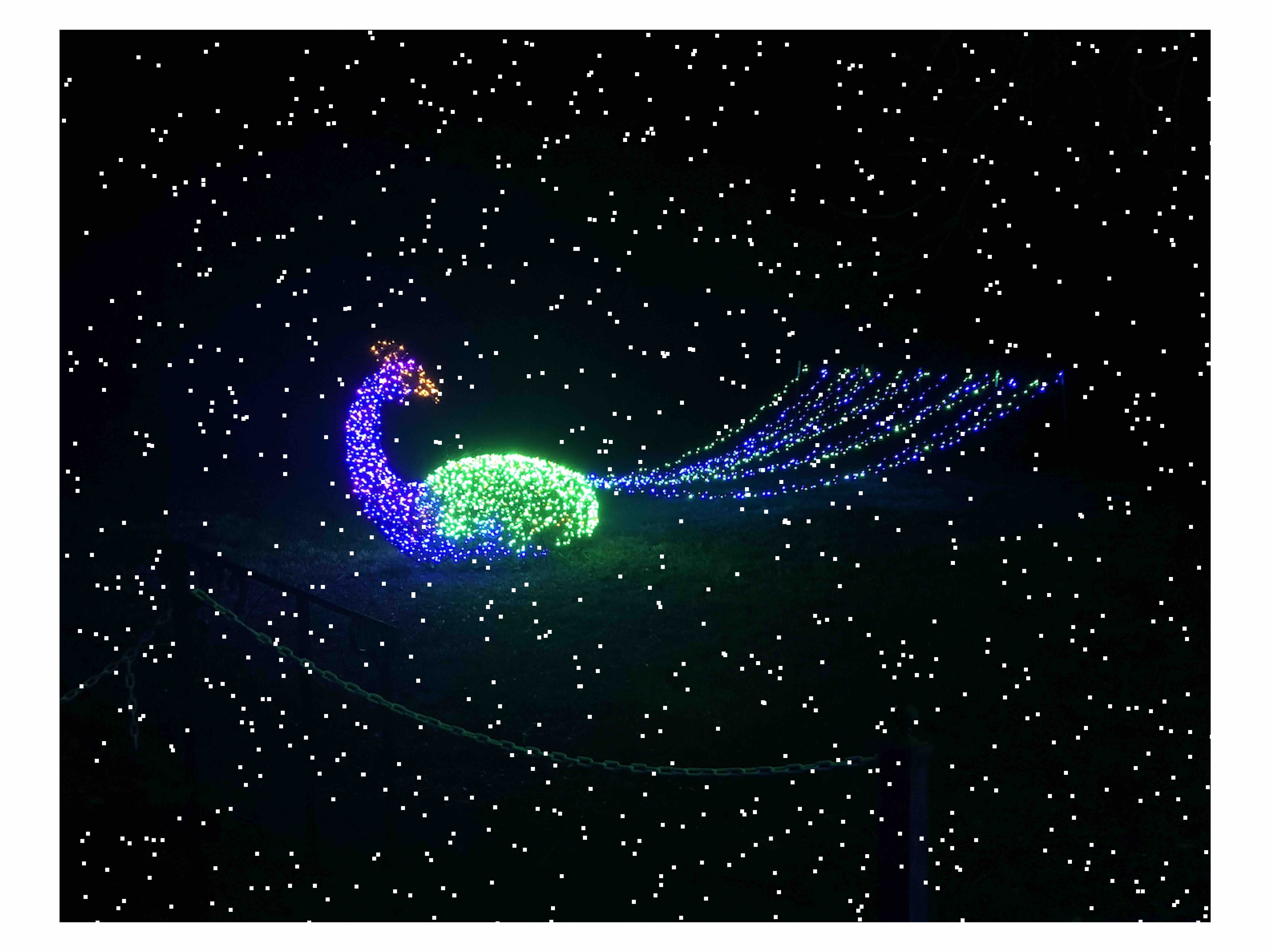}
    \includegraphics[width=0.2\linewidth]{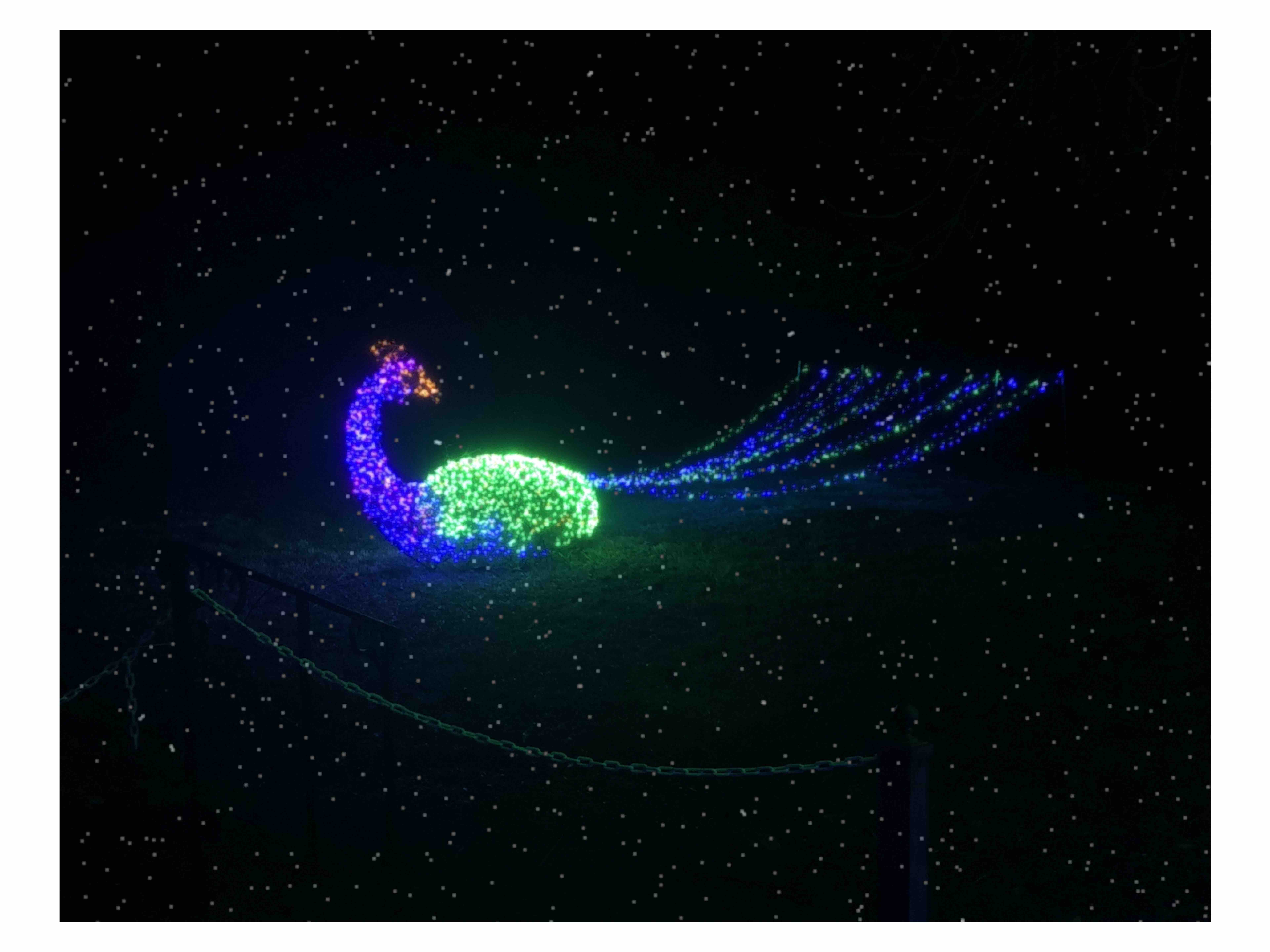}
    \includegraphics[width=0.2\linewidth]{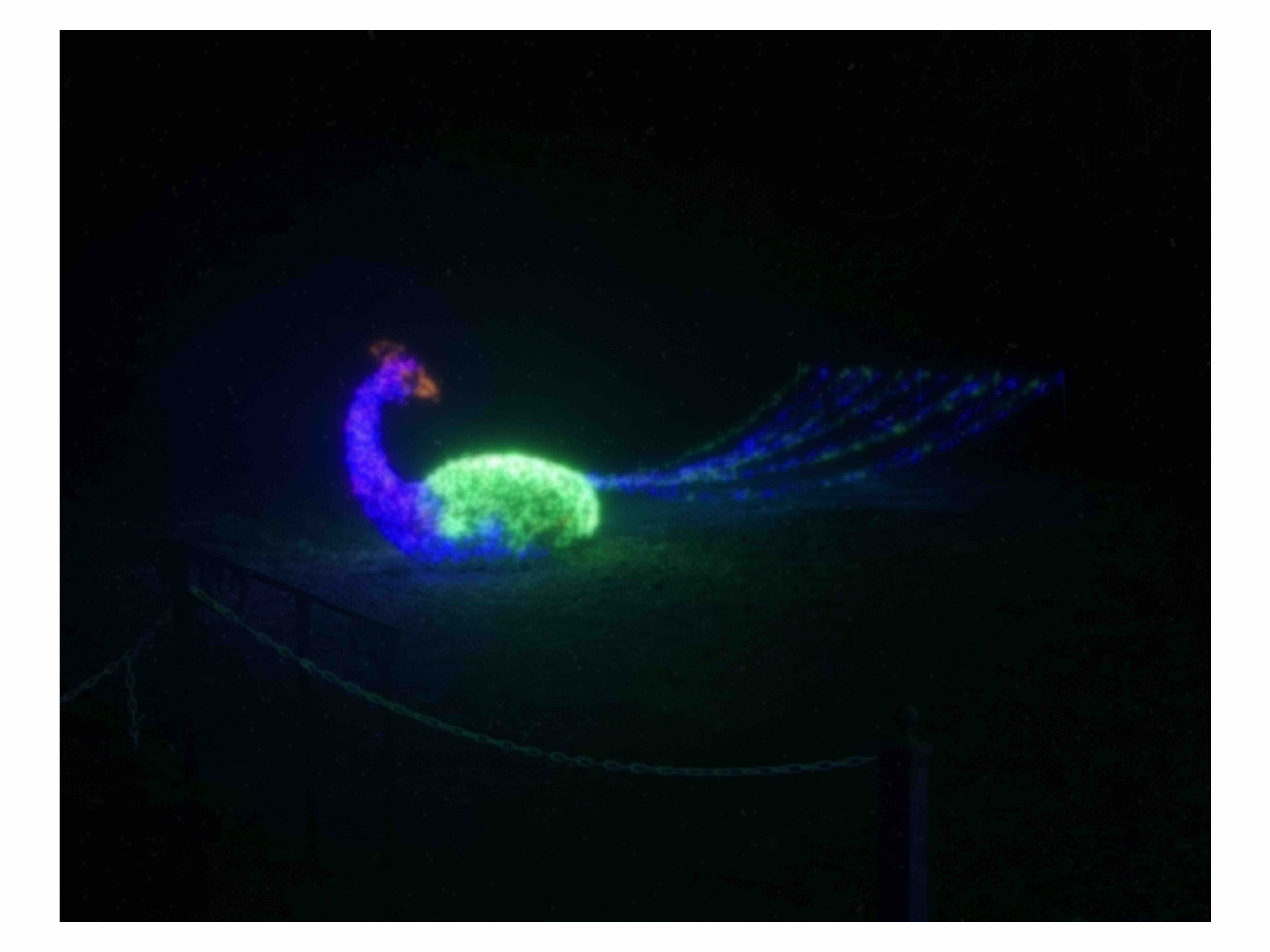}

    \caption{A no trust algorithm applied to an image. Here, approximately 10\% of pixels get corrupted by salt and pepper noise (left). Critically, the algorithm has no explicit knowledge of where. Parameters of $\tau = 10^{-5}, 10^{-6}$ (middle, right) are chosen and paired with LASSO regression.}
\end{figure}

\section{Experiments \& Applications}

\subsection{Gaussian Noise on an Image} 

We first consider  1000 images belonging to the CIFAR-10 data set modeled as signals on $32 \times 32$ grid graph $\Gr$; we use the convention of treating each pixel as a vertex connected to adjacent pixels. Importantly, we note that images \emph{are not} our primary motivation. We include this example primarily to allow visualization of our method before proceeding to more complex graphs. For a fixed image, we add Gaussian noise with different variances $\sigma^2$. We then apply the filter proposed in Theorem \ref{thm: Gaussian}, and consider the $\ell_2$ norm between the restored signal and the ground truth. We compare to two other algorithms: local averaging, and weighted nuclear norm minimization. For the local averaging, we repeatedly set the value of a vertex to be the average of its neighbors for some number of iterations $t$. Note that this is equivalent to applying the powered diffusion operator \cite{coifman2006diffusion} (or equivalently the random walk matrix) to the noisy signal $\gn$.  The nuclear norm minimization based estimate is parameterized by $\tau$, and is given by the solution to $\arg\min_{\fn} \frac{1}{2}\|\fn - \gn\|^2 + \tau \|\fn\|_\star$,
where $\|\fn\|_\star$ is the nuclear norm of $\fn$ viewed as a matrix. The penalty $\tau$ corresponds to a convex relaxation of a low-rank penalty and is designed with the assumption that noise exists over excess left \& right singular vectors of $\gn$. For the spectral estimate, we use the method of moments estimate for $2\kappa \sigma^2$ given by Equation \ref{gauss_strength}. We then calculate, for every image, the restored signal using each possible $t$ and $\tau$. The average percent error is provided in the Table \ref{tab: gaussian denoising}. We see that in the high-noise setting, the spectral denoising algorithm outperforms both local averaging and the nuclear norm estimate as a consequence of our low frequency prior.
\begin{tiny}
\begin{table}[!htbp]
\caption{Total percent error ($\|\fn_{true} - \fn_{map}\|/\|\fn_{true}\|^2$) for each combination of $\sigma, t, \tau$.
}
\label{tab: gaussian denoising}
\centering
\begin{tabular}{lrrrr}
\toprule
{} &    $\sigma= 5$ &    $\sigma = 25$ &  $\sigma = 50$ &   $\sigma = 100$ \\
\midrule
Ours                 &  11.5\% &  18.5\% &  \textbf{18.4\%} &  \textbf{22.6\%}  \\
Avg. (t=1)         &   9.6\% &  \textbf{14.0} \% &  22.5\% &  41.8\%  \\
Avg. (t=2)         &  11.9\% &  14.2\% &  19.6\% &  33.2\% \\
Avg. (t=5)         &  16.5\% &  17.3\% &  19.6\% &  26.8\% \\
N.N. ($\tau$ = 1)   & \textbf{3.9\%} &  19.8\% &  39.7\% &  79.5\% \\
N.N. ($\tau$ = 25)  &   4.0\% &  17.8\% &  37.3\% &  76.9\% \\
N.N. ($\tau$ = 50)  &   5.5\% &  16.1\% &  34.9\% &  74.3\% \\
\bottomrule
\end{tabular}

\end{table}
\end{tiny}

\begin{figure}
\centering
\includegraphics[width=0.7\linewidth]{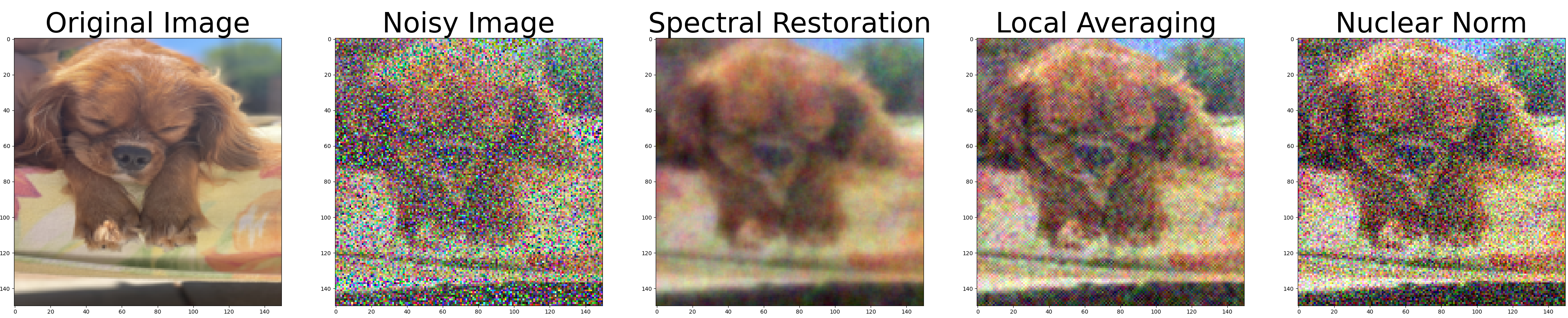}
    \caption{Best restoration on "Noodle" for the spectral, local averaging, and nuclear norm based models when $\sigma = 50$. }
    \label{fig:gauss_restorations}
\end{figure}


\subsection{High Frequency Preservation: Comparison to MAGIC}
\par We revisit MAGIC under our proposed framework. MAGIC takes our prior assumption, that the true signal is likely in the low-frequencies as a fact, rather than a probabilistic statement.  It can be interpreted as choosing $h(\lambda)$ in advance to be equal to $h(\lambda)=(1-\lambda/2)^t$ (where $t$ is a tuned parameter) rather than finding the optimal filter based combining our prior beliefs with the observed signal.  

To illustrate the advantages of our method over MAGIC, we conduct a comparison using Bernoulli dropout. We generate a set of $C = 5$ cluster centers in two dimensional space. Around each, we generate $m = 200$ points. We construct an affinity-based graph with 1000 vertices. We then consider low frequency and high frequency signals. Low frequency signals vary between clusters, while high frequency signals vary within clusters. 
Finally, we randomly set a proportion $p$ of the observations to zero for different values of $p$ and apply each algorithm. 
 Table 3 examines the resulting correlations between estimated and ground truth signals. 

\begin{figure}
\label{frequency_exp}
\centering
    \includegraphics[width=0.5\linewidth]{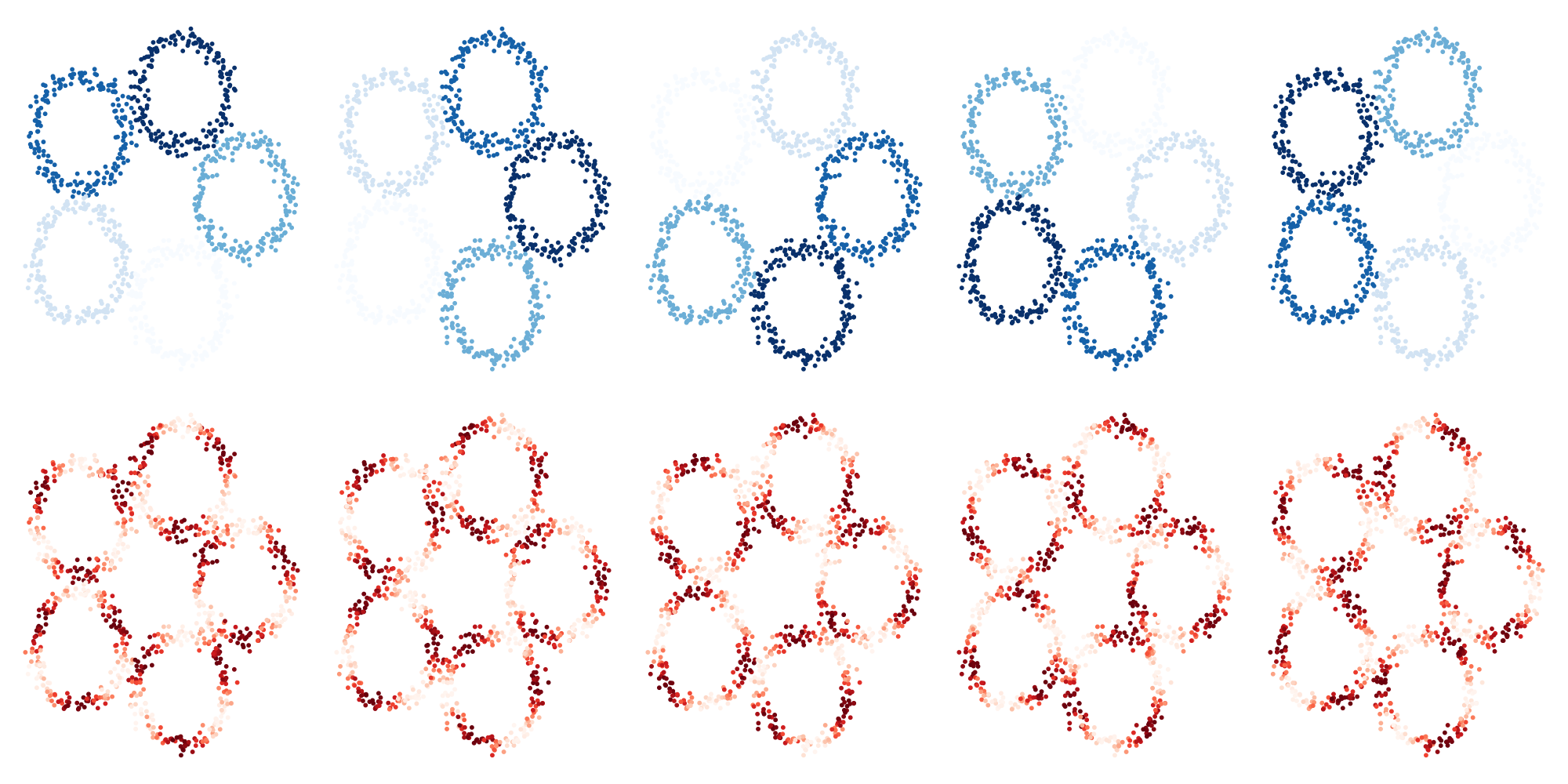}
    \caption{Top: Low frequency signals that vary over clusters. Bottom: High frequency content that periodically varies inside clusters. Each row contains different phases. }
\end{figure}
\begin{tiny}

\begin{table}[!htbp]
\centering
    \caption{Correlations between the true \& MAP signals for each frequency type, algorithm, and dropout probability $p$.}
    \begin{tabular}{lrrrr}
    \toprule
    {} &  p=0.1 &  p=0.5 &  p=0.9 &  p=0.95 \\
    \midrule
    Spectral, Low F.      &   \textbf{1.00} &   \textbf{1.00} &   \textbf{0.97} &  \textbf{0.91} \\
    MAGIC (t=1), Low F    &   0.59 &   0.57 &   0.46 &    0.38 \\
    MAGIC (t=5), Low F.   &   0.99 &   0.96 &   0.75 &    0.52 \\
    MAGIC (t=10), Low F.  &   0.55 &   0.53 &   0.42 &    0.34 \\
    \hline
    Spectral, High F.     &   \textbf{0.87} &   \textbf{0.82} &   \textbf{0.55} &   \textbf{0.41} \\
    MAGIC (t=1), High F   &   0.47 &   0.46 &   0.38 &    0.34 \\
    MAGIC (t=5), High F.  &   0.83 &   0.77 &   0.51 &    0.39 \\
    MAGIC (t=10), High F. &   0.43 &   0.41 &   0.35 &    0.27 \\
    \bottomrule
    \end{tabular}

\end{table}
\end{tiny}
Our algorithm consistently outperforms MAGIC, and the effect is most notable for high-frequency signals. This can be explained, at a high level, by the fact that the powered diffusion operator \cite{coifman2006diffusion} rapidly depresses high-frequency information, which 
 our algorithm is better able to preserve.

\subsection{Denoising Simulated single-cell Data}
\begin{tiny}
\begin{table}[!htbp]
	\centering
	\caption{Percentage error Simulated single-cell Data (lower is better)}
	\label{tab:denoise_results}
	\begin{tabular}{lrr}
		\toprule
		 & Bernoulli & Uniform \\
		\midrule
		Noisy & 50.0\% & 28.1\% \\
		Ours & \textbf{7.9\%} & \textbf{25.5\%} \\
		Local Avg & 39.9\% & 29.3\% \\
		Low Pass & 49.7\% & 28.8\% \\
		High Pass & 100.4\% & 100.3\% \\
		MAGIC & 40.1\% & 30.3\% \\
		\bottomrule
	\end{tabular}
\end{table}
\end{tiny}

\par We next apply our method to single-cell RNA sequence (scRNA-seq).
scRNA-seq data involves counting mRNA molecules in each cell, which is prone to two types of noise which we test our method's ability to remove:
\begin{enumerate}
    \itemsep 0em
    \item Bernoulli dropout. Because of the small number of the mRNAs molecules in the cell, there can be Bernoulli dropouts when the mRNAs are present but not captured by the experiment equipment  \cite{li2018accurate}.
    \item Uniform noise. For a given gene, considering the fact that the failure of mRNA capturing does not happen for all the mRNAs, but only for a percentage, we  model the noise as uniform - the counts are randomly reduced by a uniformly-distributed percentage \cite{van2018recovering}.
\end{enumerate}
\par From the data matrix, we build a nearest neighbor graph \cite{van2018recovering} where each vertex is a cell (modeled as a row vector of gene counts). A column of the matrix (a gene's counts on all the cells) is considered a signal on the graph that we can denoise with our models. By applying the models on all the columns, we obtain a matrix of the denoised data. We compare the denoising performance of our method with four existing methods. On top of the ground truth, we add different types of noise and then compare the performance of our method with MAGIC and other denoising methods: low-pass filter and high-pass band-limit filters defined w.r.t. $\Lap$, and local averaging which is a 1-step random walk on the graph using the row-normalized adjacency matrix. We compute the relative error of the denoised signals with the ground truth. In order to be able to assess the efficacy of our method, we use the bulk gene expression data of \textit{C. elegans} containing 164 worms and 2448 genes \cite{francesconi2014effects} to simulate the ground truth single-cell data, because it does not have the zero-inflation as in noisy scRNA-seq data. As shown in Table \ref{tab:denoise_results} we are better able to recover the true signal than the competing methods.

\section{Conclusion} 
We have introduced a method that denoises high-dimensional data by building a graph, treating the features as signals on the graph, and doing M.A.P estimation to recover the true features as denoised signals. We only rely on a mild prior of smoothness on the graph, making our model general and applicable to a vast variety of data modalities. We produce estimators and efficient algorithms for three types of noise common in real data: Gaussian, uniform-scaling, and Bernoulli dropout. Our model outperforms MAGIC and other methods, thanks to the modeling of the noise.


\bibliographystyle{IEEEbib}
\bibliography{refs}

\onecolumn

\appendix

\subsection{Proof of Theorem \ref{thm: Gaussian}}\label{appdx:pfthm1}
\begin{duplicate}[Gaussian Denoising]
    Suppose we observe some value $\gn$ and let $\Omega_{\gn} = \{ \fn : \boldsymbol{\widehat{f}}(\lambda_1) =  \widehat{\gn}(\lambda_1)\}$. Suppose also that $\gn$ is generated from $\fn$ by additive Gaussian noise with variance $\sigma^2$ in frequencies $\lambda_2\ldots \lambda_n$ and that $\fn$ has density $p_\kappa$ in $\Omega_{\gn}$. Then if we let $h(\lambda_i) = \frac{1}{1 + 2\kappa \sigma^2 \lambda_i}$ be the filter function, the the maximum a posteriori likelihood of $\fn$ given $\gn$ is, 
    \[ \fn_{\text{map}} = h(\gn) \]
    Furthermore, the MAP can be computed in time $\tilde{\mathcal O}\big(m \log(\epsilon^{-1}) \min\big\{\sqrt{\log(n)}, \sqrt{\frac{\lambda_{\max} + 1/2\kappa\sigma^2}{\lambda_{\min} + 1/2\kappa\sigma^2}} \big\}\big)$ to $\epsilon$ accuracy in the $\Lap$-norm.
\end{duplicate}
\begin{proof}
We first prove that the MAP is produced by the given filter. Notice first that the following expression holds for the likelihood of $\gn$ given $\fn$, up to a multiplicative constant:
\begin{align*} p(\gn | \fn) & \propto  \prod_{i=2}^n \exp\big(- \frac{1}{2\sigma^2} (\widehat{\gn}(\lambda_i) - \boldsymbol{\widehat{f}}(\lambda_i))^2\big) & \text{independence} \\
&= \exp\big(- \frac{1}{2\sigma^2} \sum_{i=2}^n (\widehat{\gn}(\lambda_i) - \boldsymbol{\widehat{f}}(\lambda_i))^2\big) & \text{log properties}  \\
& = \exp\big(- \frac{1}{2\sigma^2} \sum_{i=1}^n (\widehat{\gn}(\lambda_i) - \boldsymbol{\widehat{f}}(\lambda_i))^2\big) & \text{$\widehat{\gn}(\lambda_1) = \boldsymbol{\widehat{f}}(\lambda_1)$ when $\fn \in \Omega_{\gn}$}\\
& = \exp\big( - \frac{1}{2\sigma^2} \|\boldsymbol{\widehat{f}}-\widehat{\gn}\|^2 \big)  \\
& = \exp\big( -\frac{1}{2\sigma^2} \big) \|\fn-\gn\|^2 & \text{Parseval's Identity} 
\end{align*}
We can multiply this by the \emph{a priori} probability of $\fn$ to obtain the \emph{a posteriori} likelihood $p_\kappa(\fn | \gn)$:
\begin{align*} p_\kappa(\fn | \gn) & \propto p_\kappa(\fn) \exp\big( - \|\fn-\gn\|^2 \big) \\ 
& = \exp\big( - \kappa \fn^\top \Lap \fn - \frac{1}{2\sigma^2} \|\fn-\gn\|^2 \big)
 \end{align*}
Of course, this relation holds among those $\fn$ whose means are the same as $\gn$. To remedy this, we simply parameterize the set of $\fn$ who share a mean with $\gn$. Let $\Pi$ be the projection matrix onto $\text{span}\{\boldsymbol{1}\}^\perp$. We can write $\fn = \fn_1 + \fn_2$ where $\fn_2 = a\boldsymbol{1}$ and $a = \frac{1}{n} \boldsymbol{1}^\top \gn$. Then $\fn_2 = (\Id - \Pi)\gn$. In this case, $\fn^\top \Lap \fn = \fn_1^\top \Lap \fn_1$. Furthermore, $\|\fn - \gn\|^2 = \|\fn_1 - \Pi \gn\|^2$. Thus, by this observation and monotonicity of the exponent, it suffices to minimize the following loss function in $\fn_1$, 
\[ \mathcal L(\fn_1) =  \frac{1}{2\sigma^2} \|\fn_1 - \Pi \gn\|^2 + \kappa \fn_1^\top \Lap \fn_1  \]
Taking a gradient, 
\[ \nabla\mathcal L = \frac{1}{\sigma^2}(\fn_1 - \Pi \gn) + 2\kappa \Lap \fn_1 = \bigg(\frac{1}{\sigma^2}\Id +  2\kappa \Lap\bigg) \fn_1 - \frac{1}{\sigma^2}\Pi \gn  \]
Setting the gradient equal to zero, we obtain the necessary expression for $\fn_1$:
\[ \fn_1 =  \bigg(\Id +  2\kappa\sigma^2 \Lap\bigg)^{-1}\Pi \gn \]
If we plug in our value for $\fn_2$ and recognize that the operation $(\frac{1}{\sigma^2}\Id +  2\kappa \Lap\bigg)^{-1}$ is mean preserving, then we finally obtain the following expression:
\begin{align*} \fn &=  \bigg(\Id +   2\kappa\sigma^2\Lap\bigg)^{-1}\Pi \gn +  (\Id - \Pi)\gn \\
 &=  \bigg(\Id +  2\kappa\sigma^2 \Lap\bigg)^{-1} \gn
\end{align*}
We now prove that this is produced by the given filter. Let $\Lap = \Psi \Lambda \Psi^\top$ be the eigendecomposition of $\Lap$. Then we write,
\begin{align*} \bigg(\Id +  2\kappa\sigma^2 \Lap\bigg)^{-1} & = \bigg(\Psi \Psi^\top + 2\kappa \sigma^2 \Psi \Lambda \Psi^\top  \bigg)^{-1} & \Psi\Psi^\top = \Id \\ 
& = \bigg(\Psi\big(\Id + 2\kappa\sigma^2 \Lambda \big)\Psi^\top \bigg)^{-1} \\
& = (\Psi^\top)^{-1}\big(\Id + 2\kappa\sigma^2 \Lambda \big)^{-1}\Psi^{-1} \\ 
& = \Psi \text{diag}\big(1 + 2\kappa\sigma^2\lambda_i : i \in 1 \ldots n \big)^{-1} \Psi^\top \\ 
& = \Psi \text{diag}\bigg(\frac{1}{1 + 2\kappa\sigma^2\lambda_i} : i \in 1 \ldots n\bigg) \Psi^\top\\
& = \Psi h(\Lambda)\Psi^\top
\end{align*}
Which demonstrates that the claimed filter. To prove that the filter can be computed efficiently, we make an appeal to the solver of \cite{cohen2014solving}. Note that to compute
\[ \fn = \bigg(\Id +  2\kappa\sigma^2 \Lap\bigg)^{-1} \gn \]
It suffices to solve, 
\[ \bigg(\Id +  2\kappa\sigma^2 \Lap\bigg)\fn = \gn \] 
This can be done in time $\mathcal{\tilde{O}}(m\sqrt{\log(n)})$ because $\Id +  2\kappa\sigma^2 \Lap$ belongs to the class SDDM$_0$. This is true because i. rescaling a Laplacian by a positive constant is still a Laplacian and ii. matrices of the form Laplacian + Diagonal belong to SDDM$_0$.
\end{proof}

\subsection{The Gaussian Parameter Estimate}\label{parameter_estimation}
Suppose that $\gn$ is generated in the way that is described. For brevity, let $a = \frac{1}{\sqrt{n}}\boldsymbol{1}^\top \gn$ be the component of $\gn$ in $\bpsi_1$. Note that when $\fn \in \Omega_{\gn}$, then the distribution of $\boldsymbol{\widehat{f}}(\lambda_i)$ is $\mathcal N(0, \frac{1}{2\kappa\lambda_i})$. Furthermore, the distribution of $\boldsymbol{\widehat{f}}(\lambda-i) - \widehat{\gn}(\lambda_i)$ is $\mathcal N(0,\sigma^2)$ (it is known that if a signal is $\mathcal N(0,\sigma^2\Id)$ distributed, then its distribution is preserved under rotations). And thus, the distribution of $\widehat{\gn}(\lambda_i)$ is $\mathcal N\big(0, \sigma^2 + \frac{1}{2\kappa \lambda_i}\big)$. In this case, we can compute a moment of $\gn$ with respect to $\Lap$. 
\begin{align*} \mathbb E[\gn^\top \Lap \gn] & = \mathbb E\bigg[\sum_{i=2}^n\lambda_i \widehat{\gn}(\lambda_i)^2 \bigg] \\ 
& = \sum_{i=2}^n\lambda_i \mathbb E[\widehat{\gn}(\lambda_i)^2] & \text{linearity of expectation} \\ 
& = \sum_{i=2}^n\lambda_i \bigg( \sigma^2 + \frac{1}{2\kappa \lambda_i}\bigg) & \text{normal distribution properties} \\ 
& = \sum_{i=2}^n \sigma^2 \lambda_i + \frac{1}{2\kappa} \\ 
& = \sigma^2 \tr(\Lap) + \frac{1}{2\kappa}(n-1)
\end{align*}
Similarly, we compute again, 
\begin{align*} \mathbb E[\gn^\top \Lap^2 \gn] & = 
\sum_{i=2}^n \lambda_i^2  \bigg( \sigma^2 + \frac{1}{2\kappa \lambda_i}\bigg)  \\ 
& = \sum_{i=2}^n \lambda_i^2 \sigma^2  + \frac{\lambda_i}{2\kappa} \\  
& = \sigma^2 \tr(\Lap^2) + \frac{1}{2\kappa}\tr(\Lap)
\end{align*}
The premise of our method of estimating $\tau$ by $\hat{\tau}$ is finding values of $\hat{\kappa}$ and $\hat{\sigma^2}$ for which these equations coincide with our observations and setting $\hat{\tau} = 2\hat{\kappa}\hat{\sigma}^2$. Depending on the nature of our observations, it could be by chance that $\hat{\sigma^2}$ or $\hat{\kappa}$ are nonpositive.

\subsubsection{Case 1} Consider first the case when such a set of $\kappa$ and $\sigma^2$ exist and are positive. To find $\kappa$ and $\sigma^2$ amounts to merely solving the following $2 \times 2$ system of equations,
\[ \begin{bmatrix} \tr(\Lap) & (n-1) \\ \tr(\Lap^2) & \tr(\Lap) \end{bmatrix} \begin{bmatrix} \hat{\sigma}^2 \\ (2\hat{\kappa})^{-1} \end{bmatrix} =  \begin{bmatrix} \gn^\top \Lap \gn \\ \gn^\top \Lap^2 \gn \end{bmatrix} \]
Let $\M$ denote the above matrix. We use the standard formula of a $2 \times 2$ inverse: 
\[ \frac{1}{\det(M)} \begin{bmatrix} \tr(\Lap) & -(n-1)  \\ -\tr(\Lap^2) & \tr(\Lap) \end{bmatrix}  \begin{bmatrix} \gn^\top \Lap \gn \\ \gn^\top \Lap^2 \gn \end{bmatrix}  \]
And so, 
\[ \hat{\sigma}^2 = \frac{\tr(\Lap)\gn^\top \Lap \gn - (n-1)\gn^\top\Lap^2\gn }{\det(\M)} \hspace{1cm} (2\hat{\kappa})^{-1} = \frac{\tr(\Lap)\gn^\top \Lap^2 \gn - \tr(\Lap^2)\gn^\top\Lap\gn }{\det(\M)} \]
Combining these,
\begin{align*} \hat{\tau} & = 2\hat{\kappa}\hat{\sigma}^2 \\ 
& =  2\frac{\tr(\Lap)\gn^\top \Lap \gn - (n-1)\gn^\top\Lap^2\gn }{\det(\M)} \cdot \frac{\det(\M)}{2\bigg( \tr(\Lap)\gn^\top \Lap^2 \gn - \tr(\Lap^2)\gn^\top\Lap\gn\bigg)} \\
& = \frac{\tr(\Lap)\gn^\top \Lap \gn - (n-1)\gn^\top\Lap^2\gn }{\tr(\Lap)\gn^\top \Lap^2 \gn - \tr(\Lap^2)\gn^\top\Lap\gn} \\ 
& = \frac{(n-1)(\Lap \gn)^\top \Lap \gn  - \tr(\Lap)\gn^\top \Lap \gn}{ \tr(\Lap^2)\gn^\top\Lap\gn - \tr(\Lap)(\Lap \gn)^\top \Lap \gn} \end{align*}

\subsubsection{Case 2} In this case, rather than solving the equations exactly, we may attempt to solve the system as best as possible but with positive coefficients. This amounts to doing a change of variables $x_1 = \hat{\sigma}^2, x_2 = (2\hat{\kappa})^{-1}$, computing,
\[ x_1^\star, x_2^\star = \arg\min_{(x,y) \succeq (0,0)} \M \begin{pmatrix} x \\ y \end{pmatrix} \]
We would then compute $\hat{\sigma}^2, \hat{\kappa}$ from $x_1^\star, x_2^\star$. And again, we compute $\hat{\tau}$ from $\hat{\kappa},\hat{\sigma^2}$.  Thus, it remains to find the best $x_1^\star, x_2^\star$. To compute them, we use the following proposition about the behavior of such systems in $\mathbb R^2$:
\begin{Proposition}
    Suppose $\M = [\boldsymbol{c_1} \: \boldsymbol{c_2}] \in \mathbb R^{2 \times 2}$ and  $\boldsymbol{b} \in \mathbb R^2$ both have all positive entries. Then, 
    \[ \arg\min_{\boldsymbol{x} \preceq 0} \|\M \boldsymbol{x} - \boldsymbol{b}\| = \M^{-1} \textbf{proj}_{\textbf{convex-cone}(\boldsymbol{c_1}, \boldsymbol{c_2})}(\boldsymbol{b}) \]
    And if $\boldsymbol{b} \not \in \textbf{convex-cone}(\boldsymbol{c_1}, \boldsymbol{c_2})$ then there exists an $i \in \{1,2\}$ for which, 
    \begin{align*} \textbf{proj}_{\textbf{convex-cone}(\boldsymbol{c_1}, \boldsymbol{c_2})}(\boldsymbol{b}) & =  \textbf{proj}_{\textbf{span}(\boldsymbol{c_i})}(\boldsymbol{b}) = \frac{\boldsymbol{c_i}^\top \boldsymbol{b_i}}{\|\boldsymbol{c_i}\|^2}\boldsymbol{c_i}\end{align*}
\end{Proposition}
\begin{proof}
    We first begin with the observation that the set $C := \{\M \boldsymbol{x} : \boldsymbol{x} \succeq 0\}$ is, by definition, $\textbf{convex-cone}(\boldsymbol{c_1}, \boldsymbol{c_2})$. Thus, the vector in this set closest to $\boldsymbol{b}$ will be precisely the projection of $\boldsymbol{b}$ onto $C$; denote this projection by $\boldsymbol{y}$. To reproduce the argument $\boldsymbol{x}$ from which $\boldsymbol{y} = \M \boldsymbol{x}$, we simply compute $\M^{-1}\boldsymbol{x}$. 

    It now remains to prove the second part of the claim. The first realization is that when $\boldsymbol{b}$ is not in $C$, then its projection onto $C$ will belong to the boundary of $C$, where the boundary of $C$ is comprised of $\{\alpha \boldsymbol{c_i} : \alpha \geq 0, i \in \{1,2\}\}$. Therefore, there will exist an $i$ for which $\boldsymbol{y} = \alpha \boldsymbol{c_i}$ for some $\alpha \geq 0$. We claim that not only is this the case, but for this $i$, $\boldsymbol{y}$ is the projection of $\boldsymbol{b}$ onto the span of $\boldsymbol{c_i}$. This is true for the following reason: because $\boldsymbol{c_i}$ and $\boldsymbol{y}$ both belong to the nonnegative quadrant, the angle between $\boldsymbol{c_i}$ and $\boldsymbol{y}$ is less than $90^o$. Therefore, $\boldsymbol{c_i}^\top \boldsymbol{b} \geq 0$. Equivalently, $\alpha = \boldsymbol{c_i}^\top \boldsymbol{b} / \|\boldsymbol{c_i}\|^2 \geq 0$. This proves the desired result. We point out that the particular $i$ can be determined by either i. calculating both such projections and determining which has smaller distance to $\boldsymbol{b}$ or ii. sorting $\boldsymbol{b}, \boldsymbol{c_1},$ and $\boldsymbol{c_2}$ by increasing angle they make with the positive x-axis to determine which side of $C$ it is that $\boldsymbol{b}$ lies on. 
\end{proof}

\subsubsection{Estimation Based on Many Signals \& Correctness}
Suppose instead of one signal we have $k$ independently generated signals $\gn_1 \ldots \gn_k$. We may substitute $\gn^\top \Lap^2 \gn$ and $\gn^\top \Lap \gn$ with $\frac{1}{k}\sum_{i=1}^k \gn_i \Lap^2\gn_i$ and $\frac{1}{k}\sum_{i=1}^k \gn_i \Lap\gn_i$, respectively. As the expected value of an i.i.d. average is the same as the expected value of each entry, the same analysis produces an additional estimate of $\tau$:
\[ \hat{\tau} = \frac{(n-1) \bigg(\frac{1}{k}\sum_{i=1}^k \gn_i \Lap^2\gn_i \bigg)  - \tr(\Lap)\bigg(\frac{1}{k}\sum_{i=1}^k \gn_i \Lap\gn_i \bigg)}{ \tr(\Lap^2)\bigg(\frac{1}{k}\sum_{i=1}^k \gn_i \Lap\gn_i \bigg) - \tr(\Lap)\bigg(\frac{1}{k}\sum_{i=1}^k \gn_i \Lap^2\gn_i \bigg)} \]
We remark that, as a consequence of the Law of Large Numbers, as $k \to \infty$, then $\hat{\tau}$ converges almost surely to the true $\tau$:
\begin{align*} \frac{(n-1) \bigg(\frac{1}{k}\sum_{i=1}^k \gn_i \Lap^2\gn_i \bigg)  - \tr(\Lap)\bigg(\frac{1}{k}\sum_{i=1}^k \gn_i \Lap\gn_i \bigg)}{ \tr(\Lap^2)\bigg(\frac{1}{k}\sum_{i=1}^k \gn_i \Lap\gn_i \bigg) - \tr(\Lap)\bigg(\frac{1}{k}\sum_{i=1}^k \gn_i \Lap^2\gn_i \bigg)} & \to
\frac{(n-1) \bigg(\sigma^2 \tr(\Lap^2) + \frac{1}{2\kappa}\tr(\Lap) \bigg)  - \tr(\Lap)\bigg(\sigma^2 \tr(\Lap) + \frac{1}{2\kappa}(n-1) \bigg)}{ \tr(\Lap^2)\bigg(\sigma^2 \tr(\Lap) + \frac{1}{2\kappa}(n-1)\bigg) - \tr(\Lap)\bigg(\sigma^2 \tr(\Lap^2) + \frac{1}{2\kappa}\tr(\Lap) \bigg)} \\ 
& = \frac{2\kappa \sigma^2 (\tr(\Lap^2)(n-1) - \tr(\Lap)^2) + (n-1)\tr(\Lap) - (n-1)\tr(\Lap)}{2\kappa\sigma^2 (\tr(\Lap^2)\tr(\Lap) - \tr(\Lap^2)\tr(\Lap)) + \tr(\Lap^2)(n-1) - \tr(\Lap)^2} \\ 
& = \frac{2\kappa\sigma^2 (\tr(\Lap^2)(n-1) - \tr(\Lap)^2)}{(\tr(\Lap^2)(n-1) - \tr(\Lap)^2)}\\
& = \tau
\end{align*}
\subsubsection{Calculation}
We make some final observations that will allow us to compute these observations relatively efficiently. First, we observe that $\tr(\Lap) = \tr(\D -\A) = \tr(\D) - 0$ is the total degree of the graph. Likewise, $\Lap^2$ has on its $a$th diagonal entry the squared norm of the $a$th row of $\Lap$, the off diagonal terms contribute $\sum_{(a,b) \in E}w(a,b)$ and the diagonal term contributes $\deg(a)^2$, so $\tr(\Lap) = \sum_{a}\big(\deg(a)^2 + \sum_{(a,b) \in E}w(a,b)^2\big)$.

\subsection{Additional Properties}
We can also attempt to compute the distribution of the map estimate. Note, 
\begin{align*} \fn_{\text{map}} - \fn & = \sum_{i=2}^n\bpsi_i \bigg(  \boldsymbol{\widehat{f}}_{\text{map}}(\lambda_i) - \boldsymbol{\widehat{f}}(\lambda_i) \bigg)\\ 
& = \sum_{i=2}^n \bpsi_i \frac{1}{1 + 2\kappa\sigma^2\lambda_i}\widehat{\gn}(\lambda_i) - \boldsymbol{\widehat{f}}(\lambda_i)\bigg) \\ 
& = \sum_{i=2}^n \bpsi_i\bigg( \frac{1}{1 + 2\kappa\sigma^2\lambda_i}(\boldsymbol{\widehat{f}}(\lambda_i) + z_i) - \boldsymbol{\widehat{f}}(\lambda_i) \bigg)\\ 
& = \sum_{i=2}^n\bpsi_i\bigg(  \underbrace{\boldsymbol{\widehat{f}}(\lambda_i)}_{\mathcal N(0, (2\kappa\lambda_i)^{-1})} \bigg(\frac{1}{1 + 2\kappa\sigma^2\lambda_i} - 1\bigg) + \underbrace{z_i}_{\mathcal N(0,\sigma^2)}\frac{1}{1 + 2\kappa \sigma^2\lambda_i} \bigg)\\ 
& = \sum_{i=2}^n\bpsi_i\bigg(  \underbrace{\boldsymbol{\widehat{f}}(\lambda_i)\bigg(\frac{1}{1 + 2\kappa\sigma^2\lambda_i} - 1\bigg)}_{\mathcal N\big(0, (2\kappa\lambda_i)^{-1} \big(\frac{1}{1+2\kappa\sigma^2\lambda_i}-1\big)^2\big) }  + \underbrace{z_i\bigg(\frac{1}{1 + 2\kappa \sigma^2\lambda_i}\bigg)}_{\mathcal N(0,\frac{\sigma^2}{(2\kappa\sigma^2\lambda_i + 1)^2})} \bigg)\\ 
& = \sum_{i=2}^n \bpsi_i\bigg( \underbrace{\boldsymbol{\widehat{f}}(\lambda_i)\bigg(\frac{1}{1 + 2\kappa\sigma^2\lambda_i} - 1\bigg) + z_i\bigg(\frac{1}{1 + 2\kappa \sigma^2\lambda_i}\bigg)}_{\mathcal N(0,(2\kappa\lambda_i)^{-1} \big(\frac{1}{1+2\kappa\sigma^2\lambda_i}-1\big)^2 + \frac{\sigma^2}{(2\kappa\sigma^2\lambda_i + 1)^2})} \bigg)
\end{align*}
We then compute, 
\begin{align*} (2\kappa\lambda_i)^{-1} \big(\frac{1}{1+2\kappa\sigma^2\lambda_i}-1\big)^2 + \frac{\sigma^2}{(2\kappa\sigma^2 + 1)^2} 
& = \frac{1}{2\kappa\lambda_i} \bigg( \frac{2\kappa\sigma^2\lambda_i}{1 + 2\kappa\sigma^2\lambda_i}\bigg)^2 + \frac{\sigma^2}{(2\kappa\sigma^2 + 1)^2}  \\ 
& = \frac{1}{(2\kappa\sigma^2\lambda_i + 1)^2}\bigg( \frac{(2\kappa\sigma^2\lambda_i)^2}{2\kappa\lambda_i} + \sigma^2 \bigg) \\ 
& = \frac{1}{(2\kappa\sigma^2\lambda_i + 1)^2}\bigg( 2\kappa \sigma^4 \lambda_i + \sigma^2 \bigg) \\ 
& = \sigma^2 \frac{2\kappa\sigma^2\lambda_i + 1}{(2\kappa\sigma^2\lambda_i + 1)^2}  \\ 
& = \frac{\sigma^2}{2\kappa\sigma^2\lambda_i + 1}
\end{align*}
Therefore, there exists coefficients $c_2 \ldots c_n$ each distributed $\mathcal N(0,  \frac{\sigma^2}{2\kappa\sigma^2\lambda_i + 1})$. If we compile these coefficients into a vector $\boldsymbol{c} = (0,c_2 \ldots c_n)$
\[ \fn_{\text{map}}-\fn = \sum_{i=2}^n \bpsi_i  c_i = \Psi \boldsymbol{c} \]
As $\boldsymbol{c}$ follows a multivariate normal distribution, so does $\Psi \boldsymbol{c}$. It is mean $0$ and has covariance matrix, 
\begin{align*} \text{Cov}(\Psi \boldsymbol{c}) &= \Psi \text{Cov}(\boldsymbol{c})\Psi^\top \\ 
& = \Psi \text{diag}(0,c_2 \ldots c_n) \Psi^\top \\ 
& = \Psi \text{diag}\big(0, \frac{\sigma^2}{2\kappa\sigma^2\lambda_i + 1} : i \in 2\ldots n \big) \Psi^\top \\ 
& = \sigma^2 \bigg(\Pi + 2\kappa\sigma^2 \Lap\bigg)^+
\end{align*}
From this expression, we can determine the behavior of the error for different values of $\sigma$ and $\kappa$. When $\sigma = 0$, we recover the exact signal, and the covariance matrix goes to zero. On the other hand, when $\kappa \to 0$, the covariance matrix approaches $\sigma^2 \Pi$. This is because in the limit for such $\kappa$, we rely less and less on our prior information. In such a case, the map estimate is $\fn_{\text{map}} \to \gn$, and the above formula is simply the covariance matrix of $z = \gn - \fn$.

\subsection{The Bernoulli Model}
\subsubsection{Harmonic Interpolation \& its Runtime}
It is worth beginning with an explanation of the harmonic interpolation algorithm suggested by \cite{zhu2003semi}. The authors consider, as we do, the situation in which a signal $\gn$ is known in a set $S$ and interpolated to a new signal $\fn$ defined over $S^c$ such that the total energy $\fn^\top \Lap \fn$ is minimal. The authors state that the optimal $\fn(S^c)$ is given by $\Lap(S^c,S^c)^{-1}\A(S^c,S)\gn(S)$. 

\begin{duplicate}[Restated from \cite{zhu2003semi}]\label{thm: S known}
    Suppose $S$ has at least one edge going to $S^c$. Then there exists a unique solution to $\min_{\fn \in \Omega}\fn^\top\Lap \fn$. The interpolation of $\fn$ to $S^c$ is given by
    \begin{equation}\label{eqn: fmap formula} \fn_{\text{map}}(S^c) = \Lap(S^c,S^c)^{-1}\A(S^c,S)\gn(S). 
    \end{equation}
\end{duplicate}

We emphasize that the interpolation problem does not depend on the entirety of $S$, but rather those vertices $s \in S$ which is connected by an edge to some $\bar{s} \in S^c$; these vertices are precisely $\partial S$. Therefore, it suffices to consider the subgraph $H$ induced by the vertices $S^c \cup \partial S$, which has by definition $\hat{n}$ vertices and $\hat{m}$ edges. We remark that $\A(S^c,S)\gn(S)$ may be computed in time $\mathcal O(\hat{m})$, since for each $a \in S^c$,
\[ \big(\A(S^c,S)\gn(S)\big)(a) = \sum_{(a,b) \in E : b \in S}w(a,b)\gn(b) \]
By iterating over all such vertices $a$, we count once each edge $e \in E(S^c,S)$. As this is no more than the total number of edges $\hat{m}$ in $H$, this may be done in time $\mathcal O(\hat{m})$. We now consider the matrix $\Lap(S^c, S^c)$. This can be written as a sum $\tilde{\Lap} + \tilde{\D}$, where $\tilde{\Lap}$ is the Laplacian of the graph induced by $S^c$ (a subgraph of $H$) and $\tilde{\D}$ is a diagonal matrix. This is a SDDM matrix with no more than $\mathcal O(\hat{m})$ entries. Applying the solver of \cite{cohen2014solving}, we may solve the equation $\Lap(S^c,S^c)\fn(S^c) = \A(S^c,S)\gn(S)$ in time $\tilde{\mathcal O}(\hat{m}\sqrt{\log(\hat{n})})$.

\subsubsection{Derivation of the M.A.P. Estimate}
Thankfully, most of the computations involved in the Bernoulli model are fairly elementary. Suppose we have our ``set of suspicion'' $\zeta$ and an estimate $\fn$ of the true signal. In order for $\gn$ to be produced by $\fn$, it needs to be that $\|\fn(\zeta)-\gn(\zeta)\|_0$ observations get sent to $S^c$, all independently and with probability $p$. Otherwise, $\gn(a) = \fn(a)$ with probability $1-p$ for each of $|\zeta| - \|\fn(\zeta) - \gn(\zeta)\|_0$ observations. The result is that the conditional likelihood of $\gn$ given $\fn$ is,
\[ p^{\|\fn(\zeta)-\gn(\zeta)\|_0}(1-p)^{|\zeta| - \|\fn(\zeta)-\gn(\zeta)\|_0}\]
Multiplying this by the prior probability $p_\kappa(\fn)$, we obtain the conditional likelihood of $\fn$ given $\gn$: 
\[ p_\kappa(\fn | \gn) \propto \exp(-\kappa\fn^\top\Lap\fn)p^{|\fn(\zeta)-\gn(\zeta)\|_0}(1-p)^{|\zeta| - \|\fn(\zeta)-\gn(\zeta)\|_0}  \]
And thus, 
\begin{align*} \log(p_\kappa(\fn|\gn)) & = -\kappa \fn^\top \Lap \fn + \|\fn(\zeta)-\gn(\zeta)\|_0\log(p) + (|\zeta| - \|\fn(\zeta)-\gn(\zeta)\|_0)\log(1-p) + \text{constant} \\ 
 & = -\kappa \fn^\top \Lap \fn + \|\fn(\zeta)-\gn(\zeta)\|_0(\log(p) - \log(1-p)) + \text{constant}  \end{align*}
 Thus, to maximize likelihood, we minimize, 
 \[  \kappa \fn^\top \Lap \fn - \|\fn(\zeta)-\gn(\zeta)\|_0(\log(p) - \log(1-p))  \]
 Which is equivalent to minimizing, 
  \[   \fn^\top \Lap \fn + \|\fn(\zeta)-\gn(\zeta)\|_0\underbrace{\frac{\log(1-p) - \log(p)}{\kappa} }_{\tau \text{ as defined}} \]
Again, the sign of $\tau$ is dictated by the sign of $\log(1-p)-\log(p)$, which is positive when $1-p > p$ (i.e. $p < 1/2$), negative when $1-p < p$ (i.e. $p > 1/2$) and zero when $p = 1/2$. When this is the case, the $\ell_0$ penalty term is actually negative. Thus, we benefit from insisting that $\fn(a) \in S^c$, since that increases the conditional likelihood of $\gn$ and allows us to vary $\fn$ over the smoothness term to optimality. 
\subsubsection{The Language of Ridge Regression}
We stated an optimization for the Bernoulli model in terms of the incidence matrix. We make that formal now. First, we do assume the known property that $\Lap = \B^\top \B$. If we write this out in block notation, that means, 
\begin{align*} \begin{bmatrix}
    \Lap(\zeta,\zeta) & \Lap(\zeta,\zeta^c) \\ 
    \Lap(\zeta^c,\zeta) & \Lap(\zeta^c,\zeta^c)
\end{bmatrix} & = \Lap \\ 
& = \B\B^\top \\ 
& = 
\begin{bmatrix}
    \B(E,\zeta)^\top \\ \B(E,\zeta^c)^\top
\end{bmatrix} \begin{bmatrix}
    \B(E,\zeta) & \B(E,\zeta^c) 
\end{bmatrix} \\
& = \begin{bmatrix}
    \B(E,\zeta)^\top\B(E,\zeta) & \B(E,\zeta)^\top\B(E,\zeta^c)  \\ 
    \B(E,\zeta^c)^\top\B(E,\zeta) & \B(E,\zeta^c)^\top\B(E,\zeta^c) 
\end{bmatrix} 
\end{align*}
We compare both of these blocks elementwise to equate outer products of the incidence matrix to submatarices of the Laplacian. Additionally,
\begin{align*} \fn^\top \Lap \fn & = \begin{bmatrix} \fn(\zeta)^\top & \fn(\zeta^c)^\top \end{bmatrix} \begin{bmatrix}
    \Lap(\zeta,\zeta) & \Lap(\zeta,\zeta^c) \\ 
    \Lap(\zeta^c,\zeta) & \Lap(\zeta^c,\zeta^c)
\end{bmatrix}\begin{bmatrix} \fn(\zeta) \\ \fn(\zeta^c) \end{bmatrix} \\ 
& = \fn(\zeta)^\top \Lap(\zeta,\zeta)\fn(\zeta) + 2\fn(\zeta)\Lap(\zeta,\zeta^c)\fn(\zeta^c) + \fn(\zeta^c)\Lap(\zeta^c,\zeta^c)\fn(\zeta^c) \\ 
& = \fn(\zeta)^\top \B(E,\zeta)^\top \B(E,\zeta) \fn(\zeta) + 2\fn(\zeta) \B(E,\zeta)^\top \B(E,\zeta^c)\fn(\zeta^c) + \fn(\zeta^c)\B(E,\zeta^c)^\top \B(E,\zeta^c) \fn(\zeta^c) \\ 
& = \big(\B(E,\zeta)\fn(\zeta)-\B(E,\zeta^c)\fn(\zeta^c) \big)^\top\big(\B(E,\zeta)\fn(\zeta)-\B(E,\zeta^c)\fn(\zeta^c) \big)\\
& = \big\|\B(E,\zeta)\fn(\zeta)-\B(E,\zeta^c)\fn(\zeta^c)  \big\|_2^2 
\end{align*}
Finally, we use the fact that for all valid $\fn$, $\fn(\zeta^c) =\gn(\zeta^c)$ since $\zeta^c \subseteq S$. Combining all of these observations, it suffices to minimize, 
\[ \big\|\B(E,\zeta)\fn(\zeta)-\B(E,\zeta^c)\gn(\zeta^c)  \big\|_2^2 + \tau\|\fn(\zeta)-\gn(\zeta)\|_0 \]
Among all $\fn(\zeta)$ (and simply set $\fn(\zeta^c) =\gn(\zeta^c)$. Note that a LASSO solver might prefer an $\ell_1$ penalty term which uses the coefficients, and not their difference with another vector. For this reason, we also consider writing $\fn(\zeta) = \gn(\zeta) + \boldsymbol{x}$ for some ``difference'' variable $\boldsymbol{x}$. In this case, we can compute the best difference: 
\[ \arg\min_{\boldsymbol{x}} \|\B(E,\zeta)\boldsymbol{x} + \B(E,\zeta)\gn(\zeta) - \B(E,\zeta)\gn(\zeta^c)\|_2^2 + \tau \|\boldsymbol{x}\|_0 \]
And add the result back to $\gn(\zeta)$. It is worth that any matrix $B$ can be used such that $B^\top B = \Lap$; the square root of $\Lap$ is another logical choice. Indeed, this would size down the problem, but at much greater initial computational cost. An approximate approach for massive graphs may be to reduce the dimension of each feature using the JLT \cite{achlioptas2003database}.

\subsection{The Update Rule of the CCP}

    The algorithm of the CCP is to approximate a function of the form convex + concave by taking the first-order Taylor expansion of the concave portion; as the sum of a convex function and a linear function is convex, so is the new problem. In this case, we write, 
    \begin{align*} \mathcal L(\fn) &= \kappa \fn^\top \Lap \fn + \sum_{a \in \Vx} \log |\fn(x)| \\ 
    & =\mathcal L_{vex}(\fn) + \mathcal L_{cave}(\fn)
    \end{align*}
    Note that, 
    \[ \frac{\partial}{\partial \fn(a)}\mathcal L_{cave} = \frac{1}{|\fn(a)|} \]
    Thus, a Taylor expansion of $\mathcal L_{cave}$ about some center $\fn^t$ is, 
    \begin{align*} \widehat{\mathcal L}_{cave}(\fn ; \fn^t) & = \mathcal L_{cave}(\fn^t) + (\fn-\fn^t)^\top \nabla \mathcal L_{cave}(\fn^t) \\ 
    & = \mathcal L_{cave}(\fn^t) + \sum_{i=1}^n \frac{1}{|\fn^t(a)|}(\fn(a)-\fn^t(a)) \\ 
    & = \sum_{i=1}^n \frac{\fn(a)}{|\fn^t(a)|} + \text{constant}
    \end{align*}
    Thus, the function we try to minimize at each step of the iteration is, 
    \[ \mathcal L_{vex}(\fn)  + \widehat{\mathcal L}_{cave}(\fn ; \fn^t) = \kappa \fn^\top \Lap \fn +  \sum_{i=1}^n \frac{\fn(a)}{|\fn^t(a)|} + \text{constant} \]
    Of course, it suffices to ignore the constant terms for the purpose of optimization. Thus, the iteration is as claimed, 
    \begin{equation}\label{opt:qp} \fn^{t+1} \in \arg\min_{\Omega_{\gn}} \kappa \fn^\top \Lap \fn +  \sum_{i=1}^n \frac{\fn(a)}{|\fn^t(a)|} \end{equation}
    Because the CCP is generally a descent algorithm, this special case is a descent algorithm as well. It remains to explain why this is a Quadratic Program. Of course, the loss function is quadratic in $\fn$, so it remains to discuss the feasible region. One might recall that the feasible region is $\Omega_{\gn} = \{\fn : 0 \leq \gn(a) \leq \fn(a) \text{ or } \fn(a) \leq \gn(a) \leq 0\}$. This is a rectangular set and thus falls within the framework of QPs. \\

    Finally, we give some discussion to when $\gn(a) = 0$ exactly. This is a probability zero event according to our statistical model, and so there is some expectation that it doesn't occur. Alternatively, we can simply insist that $\fn(a) = 0$ and optimize over the remaining terms.

\section{Additional Experiments}
We validate the use of the CCP by a brief experimental comparison to projected gradient descent. To evaluate these models, we will run these algorithms on an artificial example using an image. First, we regard each color channel as a signal on a 50 $\times$ 50 grid graph. Then, we artificially generate the uniform noise, independently for each pixel, to create our signal $\gn$. We then run, for each channel, the projected gradient algorithm as well as the Convex-Concave Procedure. For the initialization, we use a bit of chemistry: although the observed signal $\gn$ is of course feasible, we perturb it each coordinate by a small amount so that $\fn^0$ is strictly feasible (heuristically, this seems to eliminate spotting in the output). For each algorithm, we use a stopping criteria that the error change by no more than $10^{-7} \times \mathcal L(\fn^0)$ (we multiply by the initial loss to provide scale). Finally, a learning rate of $\gamma = 1$ is chosen for projected gradient.

Notice that this requires knowing the ground truth signal, so this is questionable in practice. But for the purposes of experimentation, since that estimation is not the focus, this will do. Both algorithms were run in python in the Yale Zoo. In order to solve the QP ~\ref{opt:qp}, the package CVXOPT is used. 

\begin{figure}[H]
    \centering
    \includegraphics[width=0.8\linewidth]{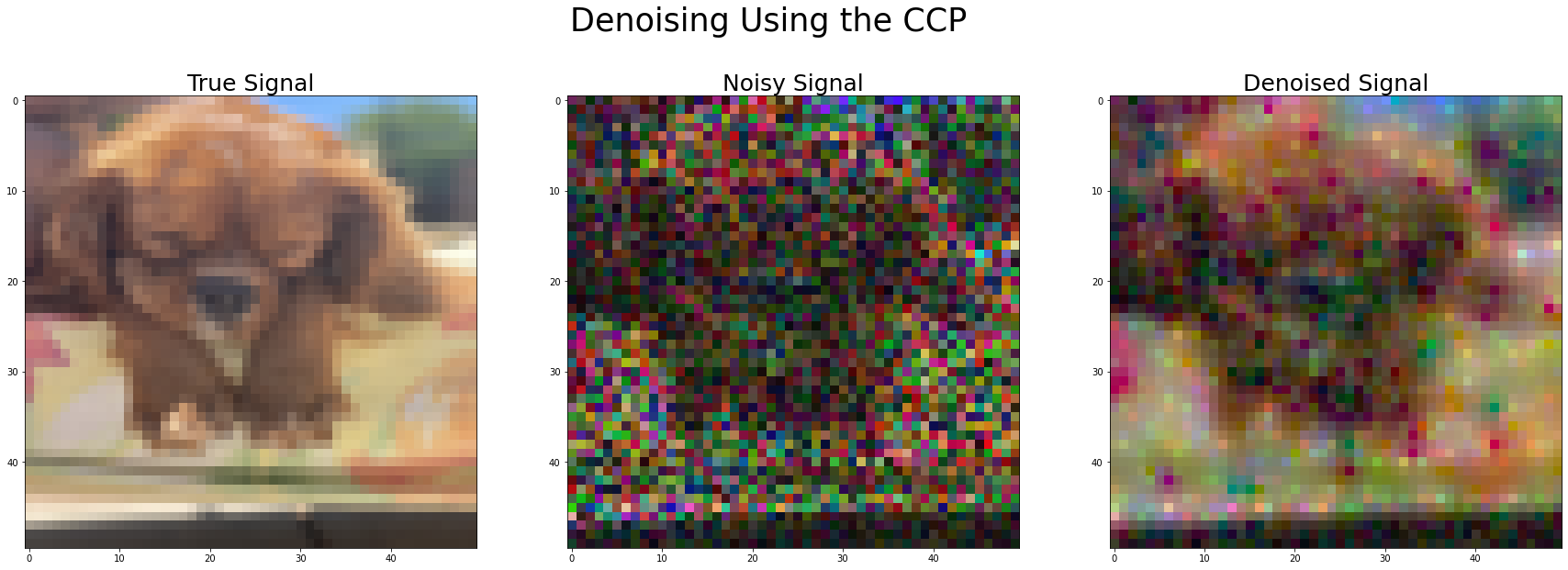}
    \includegraphics[width=0.8\linewidth]{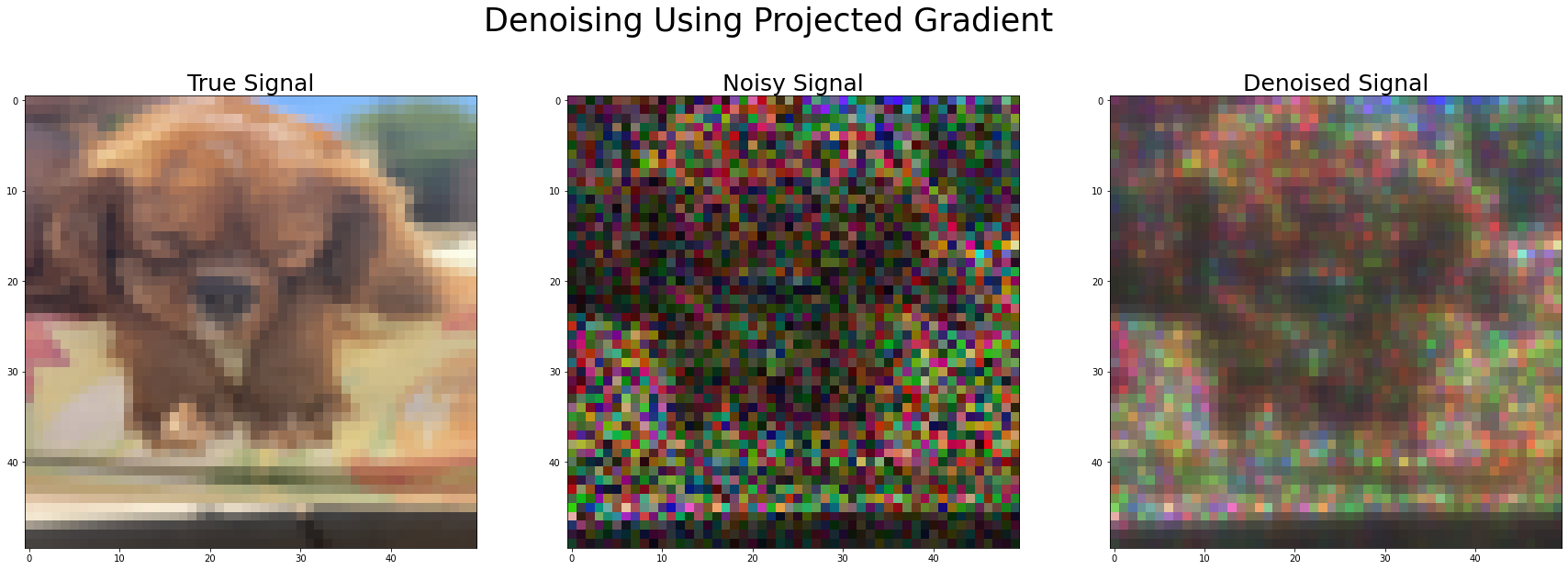}
    \caption{The original image and the randomly rescaled noisy image, followed by the output of the Convex-Concave Procedure. Below is the same experiment applied to the projected gradient method.}
    \label{fig:my_label}
\end{figure}

\begin{figure}[ht]
    \centering
    \includegraphics[width=0.4 \linewidth]{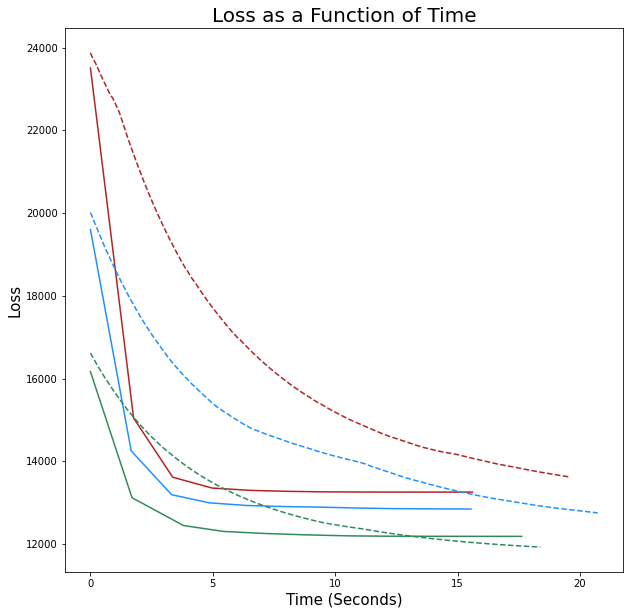}
    \caption{Value of the objective function by time. For each channel, the solid line gives the loss corresponding to the CCP at that time, while the dotted line gives the loss function corresponding to projected gradient descent.  For visual clarity, only the first 500 iterates of projected gradient are included (out of up to 1,950).}
    \label{fig:loss_curve}
\end{figure}

Note first that, theoretically, no algorithm would necessarily outperform the other. In fact, unless projected gradient were implemented with exact line search, it is not a guaranteed descent algorithm in the non-convex case. The CCP, of course, is, but we cannot make a statement about the value achieved or the rate of convergence.

\paragraph{Comparison} We compare the two algorithms in terms of runtime (both in terms of number of steps and total time), rate of decay, and accuracy. For all three channels, the CCP converged within 10 iterations (impressive)! Each iteration took 1.63 seconds on average to complete. In contrast, projected gradient descent took 1950, 1425, and 985 iterations to meet the same convergence criteria. While the iterations for the CCP are more involved than those for gradient descent, the CCP is still faster overall. We find that the CCP took 15.6, 15.5, and 17.6 seconds to converge, while projected gradient required 75.9, 54.5, and 37.5 seconds. \\ 

What is much more interesting is the rate of decay. From figure ~\ref{fig:loss_curve}, it is clear that the CCP  not only converges faster overall to its  final value, but it does so incredibly quickly compared to projected gradient, especially so for the early iterations. However, we also see that gradient descent is steadily decreasing at what appears to be an exponential rate. There is of course nothing inherent about the problem that guarantees this, although it is suggestive that the concavity due to $\log$ is dominated by the convexity of the quadratic form part of $\mathcal L$, and so our problem is ``nearly-convex.''

\vspace{0.5cm}
\begin{table}[H]
\small
\centering
\begin{tabular}{lrrrrrrrrr}
\toprule
{} &     True Obj &  Gr Obj. &      CCP Obj &      CCP Err &    Gr Err &  Gr Time &   CCP Time(s) &  Gr Time(s)  \\
\midrule
Red   &  13771.597 &      12718.518 &  13256.470 &  1298.041 &  2113.123 &   75.919 &  15.602 &       75.919  \\
Blue  &  13124.384 &      12249.471 &  12848.253 &  1182.434 &  1994.533 &   54.450 &  15.545 &       54.450  \\
Green &  12452.517 &      11747.526 &  12187.420 &  1053.541 &  1946.910 &   37.516 &  17.617 &       37.516  \\
\bottomrule
\end{tabular}
\caption{Objective function values, squared error from the true signal, and runtime associated with the projected gradient and CCP algorithm. The objective function associated to the true signal is provided as reference.}
\end{table}

Of course, the secondary question remains: how do the algorithms compare in accuracy? Here, we have two notions: final value of the objective function, and actual closeness to the ground truth signal. In all channels, the case is the same. Both algorithms pass the most natural benchmark: they are able to achieve a lower loss value than the ground truth. Therefore, differences from the true signal can be regarded as due to modeling errors or random chance, rather than poor search of the loss landscape. However, the final loss value for projected gradient is consistently lower than that for the CCP, so in this regard, projected gradient outperforms. Likewise, the $\arg\min$ from projected gradient has a lower squared error from the ground truth than the $\arg\min$ from the CCP. Given that projected gradient algorithm both outperforms the CCP in final loss and decreases smoothly, this is suggestive that the CCP, while a descent algorithm that converges fairly rapidly, is prone to getting stuck in regions of the loss landscape. It is also apparent in figure ~\ref{fig:loss_curve} that the CCP asymptotes to a sub-optimal loss, so this is not a consequence of early stopping.

\end{document}